\documentclass[11pt]{article}
\usepackage[style=alphabetic,natbib=true]{biblatex}
\usepackage{hyperref}       
\usepackage{url}    
\usepackage{bbm}
\usepackage{booktabs}       
\usepackage{amsfonts}       
\usepackage{nicefrac}       
\usepackage{enumitem}
\usepackage{ltablex}
\usepackage{multirow}
\usepackage{mathrsfs}
\usepackage{pagecolor}
\usepackage{amsmath, amssymb, amsthm, graphicx, bm}
\usepackage[ruled,vlined,linesnumbered,noend]{algorithm2e}
\usepackage{subcaption}
\usepackage{tikz} 
\usepackage{times}
\usepackage{footnote}
\usetikzlibrary{automata, positioning, arrows}
\usetikzlibrary{calc}

\tikzset{ 
node distance=1.5cm, 
every state/.style={line width=1.6pt,inner sep=0pt,minimum size=10pt},
}
\tikzset{edge/.style = {->,> = latex'}}

\usepackage{mwe}
\usepackage[section]{placeins}
\usepackage{enumitem}
\usepackage{nicefrac} 
\usepackage{mathtools}
\usepackage{cleveref}
\usepackage{thmtools,thm-restate}
\usepackage{comment}
\newtheoremstyle{definition}
  {}
  {}
  {\itshape}
  {}
  {\bfseries}
  {.}
  { }
  {\thmname{#1}\thmnumber{ #2}\thmnote{ (#3)}}

  \newtheoremstyle{theorem}
  {}
  {}
  {\itshape}
  {}
  {\bfseries}
  {.}
  { }
  {\thmname{#1}\thmnumber{ #2}\thmnote{ (#3)}}

\theoremstyle{theorem}
\newtheorem{theorem}{Theorem}
\newtheorem{proposition}{Proposition}[section]
\newtheorem{corollary}[proposition]{Corollary}
\newtheorem{lemma}[proposition]{Lemma}

\theoremstyle{definition}
\newtheorem{definition}[proposition]{Definition}

\newtheorem{remark}[proposition]{Remark}

\newcommand{\Vcal}{\mathcal{V}}

\newcommand{\EE}{\mathbb{E}} 
\newcommand{\R}{\mathbb{R}} 
\newcommand{\ZZ}{\mathbb{Z}} 
\newcommand*{\argmin}{\mathop{\mathrm{argmin}}}

\newcommand*{\one}{{\bf 1}}

\newcommand{\diag}{{\rm diag}}

\newcommand{\norm}[1]{||#1||}

\newcommand{\rbr}[1]{\left(#1\right)}
\newcommand{\sbr}[1]{\left[#1\right]}
\newcommand{\cbr}[1]{\left\{#1\right\}}
\newcommand{\nbr}[1]{\left\|#1\right\|}
\newcommand{\abr}[1]{\left|#1\right|}

\ifx\BlackBox\undefined
\newcommand{\BlackBox}{\rule{1.5ex}{1.5ex}}  
\fi

\ifx\QED\undefined
\def\QED{~\rule[-1pt]{5pt}{5pt}\par\medskip}
\fi

\ifx\proof\undefined
\newenvironment{proof}{\par\noindent{\bf Proof\ }}{\hfill\BlackBox\\[2mm]}
\fi

\newcommand{\ones}{\one^\top}
\DeclareMathOperator{\spn}{span}
\DeclareMathOperator{\rank}{rank}
\DeclareMathOperator{\poly}{poly}
\DeclareMathOperator{\domop}{dom}

\newcommand{\obs}{\mathcal{O}}

\newcommand{\dom}{\domop_+}
\makeatletter
\newcommand*{\rom}[1]{\expandafter\@slowromancap\romannumeral #1@}
\makeatother

\usepackage{fullpage}
\usepackage{xcolor}
\usepackage{algorithmic}

\usepackage{environ}
\title{Learning Hidden Markov Models Using Conditional Samples }
\author{%
  Sham M. Kakade$^1$
   \qquad
    Akshay Krishnamurthy$^2$
    \qquad
    Gaurav Mahajan$^3$
    \qquad
    Cyril Zhang$^2$\\
    \vspace{-2mm} \\
  \normalsize{$^1$Harvard University\quad $^2$Microsoft Research NYC\quad $^3$University of California, San Diego}
}
\date{\today}
\begin{document}
\maketitle

\begin{abstract}
  This paper is concerned with the computational complexity of learning the Hidden Markov Model (HMM). Although HMMs are some of the most widely used tools in sequential and time series modeling, they are cryptographically hard to learn in the standard setting where one has access to i.i.d. samples of observation sequences. In this paper, we depart from this setup and consider an \emph{interactive access model}, in which the algorithm can query for samples from the \emph{conditional distributions} of the HMMs. We show that interactive access to the HMM enables computationally efficient learning algorithms, thereby bypassing cryptographic hardness.

  Specifically, we obtain efficient algorithms for learning HMMs in two settings: 
  \begin{enumerate}[leftmargin=0.45in,rightmargin=0.3in] 
    \item An easier setting where we have query access to the exact conditional probabilities. Here our algorithm runs in polynomial time and makes polynomially many queries to approximate any HMM in total variation distance. 
    \item A harder setting where we can only obtain samples from the conditional distributions. Here the performance of the algorithm depends on a new parameter, called the fidelity of the HMM. We show that this captures cryptographically hard instances and previously known positive results. 
  \end{enumerate} 
  We also show that these results extend to a broader class of distributions with latent low rank structure. Our algorithms can be viewed as generalizations and robustifications of Angluin's $L^*$ algorithm for learning deterministic finite automata from membership queries.
\end{abstract}

\newpage
\tableofcontents
\newpage

\section{Introduction}
\label{sec:intro}
Hidden Markov Models (HMMs) are among the most fundamental tools for modeling temporal and sequential phenomena. These probabilistic models specify a joint distribution over a sequence of observations generated via a Markov chain of latent states. This structure enjoys the simultaneous benefits of low description complexity, sufficient expressivity to capture long-range dependencies, and efficient inference algorithms. For these reasons, HMMs have become ubiquitous building blocks for sequence modeling in varied fields, ranging from bioinformatics to natural language processing to finance. A long-standing challenge, in both theory and practice, is the computational difficulty of learning an unknown HMM from samples. 
In this paper, we are interested in the computational complexity of this estimation/learning task.

Although one can consider several notions of learnability, we focus on distribution learning, in total variation (TV) distance. 
In the standard realizable formulation, we are given
observation sequences generated by an underlying HMM and are asked
to efficiently compute a distribution that is close to the HMM
in TV distance. Maximum likelihood estimation is known to be statistically
efficient, but no computationally efficient implementations of this approach are known.
Indeed, 
HMMs can encode the parity with noise
problem \citep{mossel2005learning}, which is widely believed to be
computationally hard~\citep{blum1994cryptographic,kearns1994learnability,alekhnovich2003more}, and so we do not expect to find efficient algorithms for general HMMs. Recent works have therefore focused on obtaining computationally efficient algorithms under structural
assumptions which evade these hard
instances~\citep{cryan2001evolutionary,hsu2012spectral,kontorovich2013learning,weiss2015learning,huang2015minimal,sharan2017learning}.

This work takes a different perspective. We ask: can we evade computational hardness
by allowing the learner to access the HMM \emph{interactively}?
Specifically we consider a \emph{conditional sampling oracle}: we 
allow the learner to sample a ``future sequence'' from the HMM
conditioned on a ``past sequence'' or history.
This approach is closely related to recent work in distribution
testing~\citep[e.g.,][]{chakraborty2013power,canonne2014testing,canonne2015testing, bhattacharyya2018property,chen2021learning},
which demonstrates improvements in various property testing tasks via
conditional sampling. One conceptual difference is that we use
conditional sampling to evade computational hardness, rather than
obtaining statistical improvements.

From a practical perspective, we are motivated by potential
applications of interactive learning to training language models or
world models more generally. Indeed, it is quite natural to fine-tune
a language model by asking annotators to complete prompts
generated by the model; this precisely corresponds to
conditional sampling if we view the annotators as
representative of the population~\citep{zhang2022survey}. When training world models for
decision making, it may be possible to request expert demonstrations
starting from a particular state, which again is effectively
conditional sampling. This latter approach is closely related to
interactive imitation learning~\citep{ross2011reduction}.

We are further motivated by two theoretical considerations. 
First, it is not hard to show that parity with noise can be
efficiently learned in this model, as we can sample the label
conditioned on each history with a single observation set to $1$ and naively denoise these samples (we
describe this in detail in~\Cref{sec:examples}).  
However, this approach is quite tailored to noisy parity, and so it is
natural to ask if it can be generalized to arbitrary HMMs.  Second,
learning HMMs with conditional samples can be seen as a statistical
generalization of learning deterministic finite automata (DFAs) with
membership queries, for which Angluin's seminal $L^*$ algorithm
provides a strong computational separation between interactive and
non-interactive PAC learning~\citep{angluin1987learning}. We believe
it is natural to ask if $L^*$ can be extended to HMMs and be made
robust to sampling, thereby providing further evidence for the
computational benefits of interactive learning. 
\paragraph{Contributions.} 
In this paper, we develop new algorithms and techniques for learning
Hidden Markov models when provided with interactive access. As our
first result, we show how a generalization of Angluin's $L^*$
algorithm can efficiently learn any HMM in the stronger access model
where the learner can query for exact conditional probabilities. As
our main result, we consider the more natural conditional sampling
access model and obtain an algorithm that is efficient for all HMMs
with ``high fidelity,'' a new property we introduce. We show that this
property captures the cryptographically hard instances and prior
positive results, but we leave open the question of efficiently
learning all HMMs via conditional sampling. Our results require a
number of new algorithmic ideas and analysis techniques, most notably: an efficient representation for distributions over exponentially
large domains and a new perturbation argument for mitigating error
amplification over long sequences.  We hope these techniques find
application in other settings.

\subsection{Preliminaries}
\paragraph{Notation.}
Let $\obs := \{1,\ldots,O\}$ denote a finite observation space and let $\obs^t, \obs^{\leq t}$ and $\obs^*$ denote observation sequences of length $t$, observation sequences of length $\leq t$ and observation sequences of arbitrary length respectively. We
consider a distribution $\Pr[\cdot]$ over $T$ random variables
$\textbf{x}_1,\ldots,\textbf{x}_T$ with a sequential ordering, and we use $x_t \in \obs$
to denote the value taken by the $t^{\mathrm{th}}$ random
variable. For convenience, we often simply write
$\Pr[x_1,x_2,\ldots,x_T]$ in lieu of $\Pr[\textbf{x}_1{=}x_1,\ldots,\textbf{x}_T{=}x_T]$,
omitting explicit reference to the random variables themselves. 

When considering conditionals of this distribution, we \emph{always}
condition on assignment to a prefix of the random variables and
marginalize out a suffix. For example, we consider conditionals of the
form $\Pr[\textbf{x}_{t+1}{=}x_{t+1},\ldots,\textbf{x}_{t+k}{=}
x_{t+k} | \textbf{x}_1{=}x_1,\ldots,\textbf{x}_t{=}x_t]$, and we will
write this as $\Pr[x_{t+1},\ldots,x_{t+k} |
x_1,\ldots,x_t]$. Similarly, when considering tuples $f :=
(x'_1,\ldots,x'_k) \in \obs^k$ and $h :=
(x_1,\ldots,x_t) \in \obs^{t}$, we write $\Pr[\textbf{x}_{t+1} {=}
x'_1,\ldots,\textbf{x}_{t+k} {=} x'_k| \textbf{x}_1 {=}
x_1,\ldots,\textbf{x}_t {=} x_t]$ as $\Pr[f | h]$, noting that the
random variables assigned to $f$ are determined by the length of $h$.

We lift this conditioning notation to sets of observation sequences in
the following manner. If $F := \{f_1, f_2, \ldots\}$ and $H :=
\{h_1,h_2,\ldots\}$ where each $f_i,h_j \in \obs^*$, 
we write $\Pr[F|
H]$ to denote the $|F|\times|H|$ matrix whose $(i,j)^{\mathrm{th}}$
entry is $\Pr[f_i | h_j]$. We allow the sequences in $F$ and $H$ to
have different lengths, but always ensure that $\mathrm{len}(f_i)
+ \mathrm{len}(h_j) \leq T$ so that this matrix is well-defined. We
refer to rows and columns of this matrix as $\Pr[f | H]$ and
$\Pr[F | h]$ respectively.\footnote{We always refer to rows,
columns, and entries of these matrices in this manner, so no confusion
arises when constructing these matrices from (unordered) sets of
sequences.}

Lastly, for $h = (x_1,\ldots,x_t)$ we use $ho =
(x_1,\ldots,x_t,o)$ to denote concatenation, and we lift this notation
to sequences and sets. For instance, if $H = \{h_1,h_2,\ldots\}$ then
$Ho = \{h_1o,h_2o,\ldots\}$.

\subsubsection{Hidden Markov Models and low rank distributions}
Hidden Markov Models provide a low-complexity parametrization for
distributions over observation sequences. These models are defined
formally as follows.

\begin{definition}[Hidden Markov Models]\label{def:hmm}
    Let $\mathcal{S} := \{1,\ldots,S\}$. An HMM with
    $S \in \mathbb{N}$ hidden states is specified by (1) an initial
    distribution $\mu \in \Delta(\mathcal{S})$, (2) an emission
    matrix $\mathbb{O} \in \mathbb{R}^{O \times S}$, and (3) a state
    transition matrix $\mathbb{T} \in \mathbb{R}^{S \times S}$, and
    defines a distribution over sequences of length $T$ via:
    \begin{align}
        \Pr[x_1,\ldots,x_T] := \sum_{s_1,\ldots,s_{T+1} \in \mathcal{S}^{T+1}}\mu(s_1) \prod_{t=1}^T \mathbb{O}[x_t,s_t] \mathbb{T}[s_{t+1},s_{t}].\label{eq:hmm}
    \end{align} 
    Here $M[i,j]$ represents the $(i,j)^{\mathrm{th}}$ entry of a matrix $M$.
\end{definition}

As the name suggests, HMMs parameterize the distribution with a Markov
chain over a hidden state sequence along with an emission function
that generates observations. While this specific model is particularly
natural, our analysis only leverages a certain low rank structure
present in HMMs. To highlight the importance of this structure, we
define the \emph{rank} of a distribution.

\begin{definition}[Rank of a distribution]\label{def:lowrank}
    We say distribution $\Pr[\cdot]$ over observation sequences of
    length $T$ has rank $r$ if, for each $t \in [T]$, the conditional
    probability matrix $\Pr[\obs^{\leq T-t}|\obs^{t}]$ has rank at
    most $r$.\footnote{When some histories occur with zero probability, there might be multiple consistent conditional probability functions associated to a distribution, in which case the rank is not uniquely defined. We address this by \emph{defining} the distribution via its conditionals (which determine the rank); see \Cref{rem:rank}.}
\end{definition}

An HMM with $S$ hidden states has rank
at most $S$, which can be verified using the fact that the hidden states form a Markov
chain (we give a proof in \Cref{sec:hmmlowr}).\footnote{In fact the rank of the HMM can be much smaller, since
the decomposition alluded to above realizes the non-negative rank of
the matrix, which can be exponentially larger than the rank.} More
generally, the rank identifies a low dimensional structure in the
distribution: we have exponentially many vectors $\Pr[\obs^{\leq
T-t} | h]$, one for each history $h$, in an $r$-dimensional subspace of
an exponentially larger ambient space.  Thus, we are interested in
algorithms that exploit the low dimensional structure and admit
statistical and computational guarantees scaling polynomially with the rank.

\subsubsection{Learning models}
To circumvent computational hardness, we allow the learner to access
conditional distributions of the underlying distribution
$\Pr[\cdot]$. We specifically consider two access models formalized
with the following oracles: \footnote{Both oracles require committing to a consistent choice of conditional probability distribution when conditioning on zero probability events; see~\Cref{rem:rank}.}

\begin{restatable}[Exact conditional probability oracle]{definition}{defone}\label{def:exact_oracle}
    The exact conditional probability oracle is given as input: observation sequences $h$ and $f$ of length $t \leq T$ and $T-t$ respectively, chosen by the algorithm, and returns the scalar $\Pr[f|h]$.
\end{restatable}

\begin{restatable}[Conditional sampling oracle]{definition}{deftwo}\label{def:conditional_oracle}
    The conditional sampling oracle is given as input: an observation
    sequence $h$ of length $t \leq T$, chosen by the algorithm, and
    returns an observation sequence $f$ of length $T-t$ such that the
    probability that $f$ is returned 
    is $\Pr[f | h]$, independently of
    all other randomness.
\end{restatable}

When considering the exact probability oracle, we also allow the
learner to obtain independent samples from the joint distribution
$\Pr[\cdot]$. Note that this oracle equivalently provides access to
exact (unconditional) probabilities of length $T$ sequences. We view
this as a noiseless analog of the conditional sampling oracle, which
is the main model of interest.
It is analogous to noiseless oracles in the distribution testing literature \citep[e.g.,][]{canonne2014testing}.

As a learning goal, we consider distribution learning in total
variation distance as studied in prior
works~\citep{kearns1994learnability,mossel2005learning,hsu2012spectral,anandkumar2014tensor}. Given
access to a target distribution $\Pr[\cdot]$ we want to efficiently
compute an estimate $\widehat{\Pr}[\cdot]$ that is close in total
variation distance to $\Pr[\cdot]$. Formally, we want an algorithm that, when given parameters $\varepsilon,\delta>0$, computes an estimate $\widehat{\Pr}[\cdot]$ such that with probability at least $1-\delta$ we have
\begin{align*}
    \mathrm{TV}(\Pr,\widehat{\Pr}) := \frac{1}{2} \sum_{x_1,\ldots,x_T \in \obs^T} \abr{\Pr[x_1,\ldots,x_T] - \widehat{\Pr}[x_1,\ldots,x_T]} \leq \varepsilon.
\end{align*}
The algorithm is efficient if its computational complexity (and hence
number of oracle calls) scales polynomially in $r,T,O,1/\varepsilon$
and $\log(1/\delta)$. 

\begin{remark}
Note that, as the support
of $\Pr[\cdot]$ is exponentially large in $T$, it is not possible to
write down all $\obs^T$ values of $\widehat{\Pr}$ efficiently. Instead, the goal is to return 
an efficient representation from which we can evaluate
$\widehat{\Pr}[x_1,\ldots,x_T]$ for any sequence $x_1,\ldots,x_T$ efficiently.
It will become clear what constitutes an efficient representation for low rank distributions in the sequel. However for HMMs for example, 
the tuple of initial distribution $\mu$,
observation operator $\mathbb{O}$, and transition operator
$\mathbb{T}$ form an efficient representation.
\end{remark}

\subsection{Our results}
Our first result studies the computational power provided by the exact probability oracle (\Cref{def:exact_oracle}). 
We show how a generalization of Angluin's $L^*$ algorithm can
efficiently learn any HMM given access to this oracle. The result is
summarized in the following theorem:\footnote{As this result is a warmup for our main result, we focus on the setting where $\obs = \{0, 1\}$ for simplicity.}

\begin{restatable}[Learning with exact conditional probabilities]{theorem}{thmone}
    \label{thm:exact} Assume $\obs = \{0,1\}$. Let $\Pr[\cdot]$ be any
    rank $r$ distribution over observation sequences of length
    $T$. Pick any $0 < \varepsilon, \delta < 1$. Then \Cref{alg:exact} with access to an
    exact probability oracle and samples from $\Pr[\cdot]$, runs in
    $\poly(r,T,1/\varepsilon,\log(1/\delta))$ time and returns an efficiently represented approximation 
    $\widehat{\Pr}[\cdot]$ satisfying
    $\mathrm{TV}(\Pr,\widehat{\Pr}) \leq \varepsilon$ with probability at least
    $1-\delta$.
\end{restatable}


The main technical challenge is finding a succinct and observable
parametrization/representation of the distribution, so that we can infer all
conditional distributions using polynomially many queries.
This observable parameterization plays a central role in our main
result, and in this sense \Cref{thm:exact} can be seen as an
insightful warmup.


Our main contribution is in extending this result to the more
natural interactive setting where the learner only accesses
conditional samples via the oracle in~\Cref{def:conditional_oracle}. 
Our
algorithm here can be viewed as a robust version of
$L^*$, and we obtain the following guarantee:

\begin{restatable}[Learning with conditional samples]{theorem}{thmtwo}
  \label{thm:main} Let $\Pr[\cdot]$ be any rank $r$ distribution over
  observation sequences of length $T$. Assume distribution $\Pr[\cdot]$ has fidelity $\Delta^*$. Pick any $0 < \varepsilon, \delta < 1$. Then \Cref{alg:cond} with access to a conditional sampling oracle runs in
  $\poly(r,T,O,1/\Delta^*,1/\varepsilon,\log(1/\delta))$ time and returns an efficiently represented approximation
  $\widehat{\Pr}[\cdot]$ satisfying
  $\mathrm{TV}(\Pr,\widehat{\Pr}) \leq \varepsilon$ with
  probability at least $1-\delta$.
\end{restatable}

The theorem provides a robust analog to~\Cref{thm:exact} in the much
weaker conditional sampling access model. The caveat is that the
guarantee depends on a spectral property of a distribution, which we
call the fidelity. The definition of fidelity (\Cref{def:fidelity}) requires further
development of the algebraic structure in $\Pr[\cdot]$ and is deferred
to \Cref{sec:tech}. Nevertheless, we can show that the
cryptographically hard examples of HMMs and positive results from
prior work on learning HMMs have fidelity that is lower bounded by a (small) polynomial of the other parameters and
thus are efficiently learnable by our algorithm
(see \Cref{sec:examples}). On the other hand, there are HMMs with
exponentially small fidelity, and we have no evidence that
these instances are computationally intractable when provided with
conditional samples.  This leads to the main open question stemming
from our work.

\begin{restatable}{openproblem}{openone}
  \label{open}
  Is there a computationally efficient algorithm for learning \emph{any}
low rank distribution given access to a conditional sampling oracle?
\end{restatable}

\paragraph{Paper organization.} 
In \Cref{sec:tech}, we present an overview of our techniques,
explaining the challenges and how we address them. Then we turn to the
more formal presentation of the proofs, with \Cref{sec:toy} devoted
to \Cref{thm:exact} and \Cref{sec:main} devoted
to \Cref{thm:main}. These sections present our algorithms and the main
ingredients for their analysis, with some details deferred to the
appendices. We close the main body of the paper in \Cref{sec:open},
with some further discussion regarding \Cref{open}.

\section{Technical overview}
\label{sec:tech}
To explain the central challenges with learning low rank distributions
and how we overcome them,
let us introduce the following notation: let $H_t:= \obs^t$ and
$F_t := \obs^{T-t}$ denote the observation sequences of length $t$ and
$T-t$ respectively. Then the matrix $\Pr[F_t | H_t]$ is a submatrix
of $\Pr[\obs^{\leq T-t} | \obs^t]$ and hence is rank at most $r$ by
assumption. If we define these matrices for each length $t \in [T]$, then clearly we
have encoded the entire distribution. Hence, estimating these matrices
in an appropriate sense would suffice for distribution learning.
Although the matrices all have rank at most $r$, they are
exponentially large, so the low rank property does not immediately
yield an efficient representation of the distribution.  Indeed, we
must leverage further structure to obtain efficient algorithms.


\subsection{Background: Observable operators and hard instances} \label{sec:hard-back}
For HMMs, we can hope to leverage the explicit formula for the
probability of a sequence (\Cref{eq:hmm}) to obtain an efficient
algorithm. Indeed, this is the approach adopted by Hsu, Kakade, and
Zhang~\citep{hsu2012spectral}. Specifically, they use the
\emph{observable operator}
representation~\citep{jaeger2000observable}: if we define $S \times S$
matrices $\{\mathbb{A}_o\}_{o \in \obs}$ as $\mathbb{A}_o :=
  \mathbb{T}\mathrm{diag}(\mathbb{O}[o,\cdot])$ then we can write the
probability of any observation sequence as
\begin{align*}
  \Pr[x_1,\ldots,x_T] = \one^\top \mathbb{A}_{x_T}\ldots \mathbb{A}_{x_1}\mu,
\end{align*}
where $\one$ is the all-ones vector and recall that $\mu$ is the
initial state distribution.
Hsu, Kakade and Zhang show that these operators can be estimated, up
to a linear transformation, whenever $\mathbb{T}$ and $\mathbb{O}$
have full column rank. In fact, under their assumptions, these
operators can be recovered from
$\Pr[\textbf{x}_1{=}\cdot,\textbf{x}_2{=}\cdot,\textbf{x}_3{=}\cdot]$
alone; no higher order moments of the distribution are required.

Unfortunately, this approach fails if either $\mathbb{T}$ or
$\mathbb{O}$ are (column) rank deficient, and it is conjectured that
the rank deficient HMMs are precisely the hard
instances~\citep{mossel2005learning}. On the other hand, many
interesting HMMs \emph{are} rank deficient. For example, any
\emph{overcomplete} HMM---one with fewer observations than
states---cannot have a full column rank $\mathbb{O}$ matrix. This
captures all deterministic finite automata where the alphabet size is
smaller than the number of states as well as the parity with noise
problem.

Learning parity with noise is a particularly interesting case. The
standard formulation is that we obtain samples of the form
$(\textbf{z},\textbf{y}) \in \{0,1\}^{T-1}\times\{0,1\}$ where
$\textbf{z}$ is uniformly distributed on the hypercube and $\textbf{y}
  = \bigoplus_{i \in I}\textbf{z}_i$ with probability $1-\alpha$ and
$\textbf{y} = 1 - \bigoplus_{i \in I}\textbf{z}_i$ with the remaining
probability. Here $\bigoplus$ denotes the parity operation, $I$ is a
secret subset of indices $I\subseteq[T-1]$, and $\alpha \in (0,1/2)$ is a
noise parameter. We want to learn the subset $I$, given samples from
this process. This problem is widely believed to be computationally
hard and can be encoded as an HMM with $\obs = \{0,1\}$ and $4T$
states (see~\Cref{sec:examples}).  
This HMM exhibits two particularly challenging features. First, many states have identical observation distributions, or are \emph{aliased}; 
characterizing the learnability (as well as basic structural properties)
of aliased HMMs remain long-standing open problems~\citep{weiss2015learning}. Second,
it is
quite apparent that low degree moments, like those used by Hsu,
Kakade, and Zhang, reveal no information about the subset $I$. In
particular, the observable operators $\mathbb{A}_o$ are not identifiable from
low degree moments. One must use higher order information, i.e.,
statistics about long sequences, to solve this problem.

\begin{figure}[t!]
  \begin{minipage}{0.55\textwidth}
    \centering
    \includegraphics[width=\textwidth]{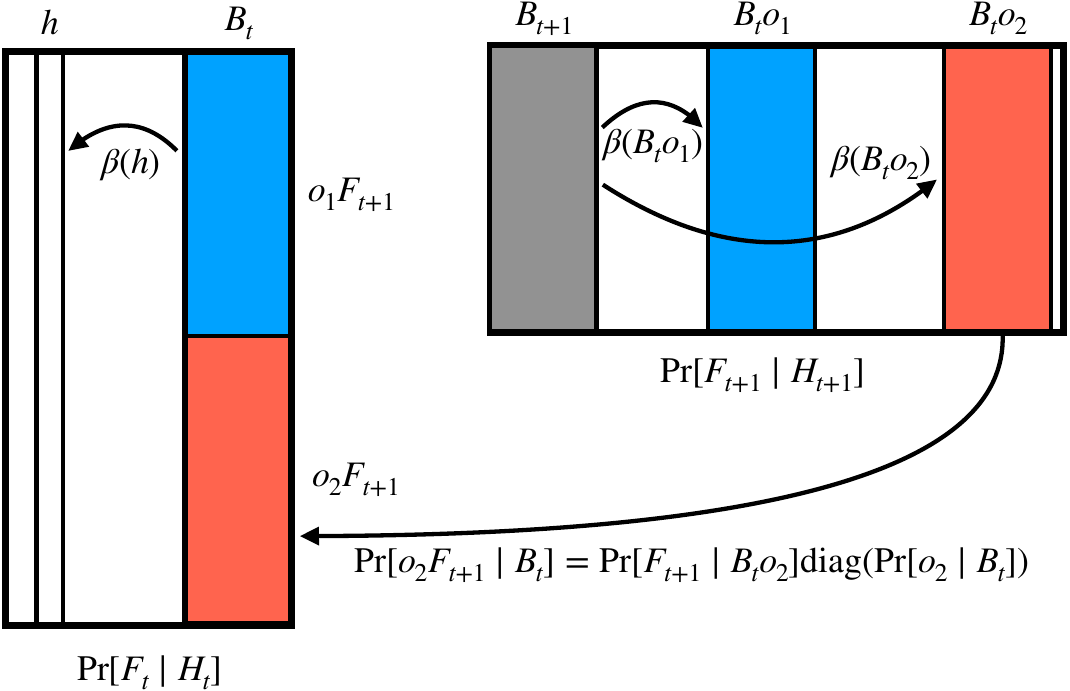}
  \end{minipage}
  \hfill
  \begin{minipage}{0.4\textwidth}
    \caption{
      Schematic of the circulant structure relating the $\Pr[F_t | H_t]$ and $\Pr[F_{t+1} | H_{t+1}]$ matrices. Columns of $\Pr[F_t \mid H_t]$ can be represented linearly in basis $B_t$ using coefficients $\beta(\cdot)$. The blocks $\Pr[oF_{t+1} \mid B_t]$ appear in the next matrix $\Pr[F_{t+1} \mid H_{t+1}]$ (up to scaling), so they can be represented in basis $B_{t+1}$, yielding operators $A_{o,t}$.}
    \label{fig:circulant}
  \end{minipage}
\end{figure}

\subsection{Efficient representation}\label{sec:eff-rep}
For rank deficient HMMs, it is not clear how to identify the operators
$\mathbb{A}_o$ and it is not even clear that such operators exist for
the more general case of low rank distributions. So, we must return to
the question of how to efficiently represent the distribution. Here,
we leverage the more general formulation of observable operators and
the equivalence between observable operator models and low dimensional
stochastic processes developed by~\citet{jaeger2000observable}.

The first observation is that any submatrix of $\Pr[F_t | H_t]$ that
has the same rank as the entire matrix can be used to build an
efficient representation. To see why, suppose we have such a
submatrix, and let us index the columns/histories of the submatrix by
$B_t$, which we refer to as the \emph{basis}. It follows that $\Pr[F_t
| B_t]$ spans the column space of $\Pr[F_t | H_t]$, which implies that
for any history $h \in H_t$ there exists coefficients
$\beta(h) \in \mathbb{R}^{|B_t|}$ such that
\begin{align*}
  \Pr[F_t | h] = \Pr[F_t | B_t]\beta(h).
\end{align*}
The main observation toward obtaining an efficient representation is
to exploit a certain circulant structure in the matrices $\{ \Pr[F_t
    | H_t]\}_{t\leq T}$ to model the evolution of the coefficients
(visualized in~\Cref{fig:circulant}). The
circulant structure is simply that for basis $B_t$, observation $o$, and
future $f \in F_{t+1}$ (i.e., of length $T-t-1$) the vector $\Pr[B_t o
    f]$ appears in two of the matrices (albeit with different
scaling). It appears in the matrix $\Pr[F_t | H_t]$ in row $of$ and
columns $B_t$, and it appears in the matrix $\Pr[F_{t+1} |
    H_{t+1}]$ in row $f$ and columns $B_to$. Thus, if we learn how to
represent the columns $\Pr[F_{t+1} | B_to]$ in terms of the
columns $\Pr[F_{t+1} | B_{t+1}]$---which we can do via the
coefficients---the circulant property provides a connection between
the matrices $\Pr[F_{t+1} | H_{t+1}]$ and $\Pr[F_t |
    H_t]$.

Formally, we can define operators
$\{A_{o,t}\}$ for each observation $o \in \obs$ and sequence length $t \in [T]$ satisfying
\begin{align}
  \Pr[F_{t+1} | B_{t+1} ]A_{o,t} = \Pr[oF_{t+1} | B_t], \label{eq:operator_full}
\end{align}
which can then be used to express sequence probabilities by iterated
application. Indeed, we have
\begin{align}
  \Pr[x_1,\ldots,x_T] & = \Pr[x_1,\ldots,x_T|B_0] = \Pr[x_2,\ldots,x_T| B_{1}]A_{x_1,0} = \ldots \notag                                 \\
  \ldots              & = \Pr[x_{T}| B_{T-1}] A_{x_{T-1},T-2}\ldots A_{x_1,0} = A_{x_{T},T-1}\ldots A_{x_1,0}\, ,\label{eq:telescoping}
\end{align}
where by an explicit choice of $B_0$, $B_T$ and $F_T$,  the matrices $A_{x_1, 0}$ and $A_{x_T, T-1}$ are column and row vectors respectively, and so the right-hand side is a scalar (see \Cref{prop:toy:1} for details)\footnote{We define $B_0$, $B_T$ and $F_T$ to be singleton sets. $B_0$ and $F_T$ contain the empty string $\varphi$ and $B_T$ contains any length $T$ observation sequence. These new definitions, in conjunction with \Cref{prop:toy:1} imply: $A_{x_T, T-1} = \Pr[x_T | B_{T-1}]$ and therefore will be a row vector. Similarly, $A_{x_1, 0}$ is a solution of $
    \Pr[F_1 | B_1] A_{x_1, 0} = \Pr[x_1 F_1 | \varphi]
  $ and is therefore a column vector.}.
More importantly, these operators can also be viewed as evolving the coefficients via the identity:
\begin{align}\label{eq:intro-coeff-2}
  \forall h \in H_t, o \in \obs: \beta(ho) = \frac{A_{o,t} \beta(h)}{\Pr[o | h]}.
\end{align}
This identity is proved in~\Cref{prop:toy:1} and~\Cref{prop:main1}. We highlight the scaling, which results in a nonlinear update equation
and appears because the coefficients express conditional rather than
joint probabilities. This viewpoint of operators evolving coefficients
will play a central role in our error analysis.

Thus, it remains to find the bases $\{B_t\}_{t\leq T}$, estimate the
operators $\{A_{o,t}\}_{o \in \obs, t \leq T}$, and control the error
amplification from iteratively multiplying these estimates. We turn to
these issues next.

\begin{remark}
  The approach of Hsu, Kakade, and Zhang can also be viewed as
  estimating operators via~\Cref{eq:operator_full} with the particular
  choice of basis. They show that conditional distribution of futures
  given any history can be written in the span of the conditional
  distributions of the single observation histories, so that $\obs$
  itself forms a basis. This is implied by their assumptions and it
  permits using only second and third degree moments to estimate the
  operators. However, in general we will need to use long sequences in
  our bases and interactive access will be crucial for
  estimation.
  Additionally, under their choice of bases and their assumptions
  they show that the solution of~\Cref{eq:operator_full} is related
  to the $\mathbb{A}_o$ operators,
  explicitly given by $\mathbb{T}$ and $\mathbb{O}$, by an
  invertible and bounded transformation, which is instrumental in their error
  analysis.
  When considering general bases $B$, we do not have such a connection and
  will require a novel error propagation argument.
\end{remark}

\begin{remark}
The above structural result identifying an efficient representation
for general low rank distribution is due
to~\citet{jaeger2000observable}. The central object of his study is
the \emph{observable operator model} (OOM) which parametrizes the
sequence probabilities of a distribution via products of
finite-dimensional matrices, analogous to~\Cref{eq:telescoping}. Note
that the $\mathbb{A}_o$ matrices are one such choice of observable
operator matrices for HMMs, but the definition is more general. He
uses the terminology ``dimension'' of a stochastic process instead of
``rank'' of a distribution. The result we have summarized above
establishes an equivalence between low-dimensional stochastic
processes and OOMs.  As this result is central to our algorithm and
analysis, we provide detailed proofs below for completeness.
\end{remark}

\subsection{Error propagation}
Although finding the bases $B_t$ and estimating corresponding operators $A_{o,t}$ is nontrivial, even if we have estimated these operators accurately,
we must address the error amplification that can arise from repeated application of the learned
operators. This challenge makes up the majority of our technical analysis. We discuss estimating operators $A_{o,t}$ in \Cref{sec:intro:estimate} and how to find the basis
in~\Cref{sec:intro:base}.


To explain the error amplification challenge, suppose for now that we are given
bases $\{B_t\}_{t \leq T}$ and subsequently estimate the operators $A_{o,t}$ in $\ell_2$ norm, i.e., we have estimate $\widehat{A}_{o,t}$ satisfying $\|\widehat{A}_{o,t} - A_{o,t}\|_2 \leq \varepsilon$. We first define our estimated model $\widehat{\Pr}$ in terms of the
estimated operators
$\widehat{A}_{o,t}$. Considering~\Cref{eq:telescoping}, the natural estimator is
\begin{align}
  \widehat{\Pr}[x_1,\ldots,x_T] = \widehat{A}_{x_{T},T-1}  \ldots \widehat{A}_{x_2,1} \widehat{A}_{x_1,0}, \label{eq:estimator}
\end{align}
where, as before, the matrices $\widehat A_{x_1, 0}$ and $\widehat A_{x_T, T-1}$ are column and row vectors respectively, so the right hand side is a scalar. To simplify notation for this section, we omit the time indexing on the operators.

Given this estimate, the total variation distance is
\begin{align*}
  \frac{1}{2} \sum_{x_1,\ldots,x_T\in \obs^T}\abr{ \widehat{A}_{x_{T}}\ldots \widehat{A}_{x_1} - A_{x_{T}}\ldots A_{x_1}}.
\end{align*}
Let us first discuss two strategies for bounding this expression that
can work in some cases, but do not seem to work in our setting. One
idea is to pass to the $\ell_2$ norm and use a telescoping argument to
obtain several terms of the form
\begin{align*}
  \sum_{x_1,\ldots,x_{T}\in\obs^{T}} \| \widehat{A}_{x_{T}}\ldots\widehat{A}_{x_{t+2}}\|_2 \cdot \|\rbr{\widehat{A}_{x_{t+1}} - A_{x_{t+1}}} A_{x_t} \ldots A_{x_1}\|_2 
\end{align*}
These terms are convenient because the matrix products only disagree
in the $t^{\mathrm{th}}$ operator. However, both the ``incoming'' product $A_{x_t}\ldots A_{x_1}$ that pre-multiplies this difference and
the ``outgoing'' product $\widehat{A}_{x_T}\ldots\widehat{A}_{x_{t+2}}$ whose norm we must bound can be rather poorly behaved. For example, the product
$A_{x_t}\ldots A_{x_1}$ can have $\ell_2$ norm that grows
exponentially with $t$, since the $\ell_2$ norm of the individual
matrices can be much larger than $1$.  An even worse problem is that
we have exponentially many terms in the sum, so that even bounding
each term by $\varepsilon$ (which would be possible if the incoming
and outgoing products were well behaved) is grossly insufficient.

The other approach is the strategy adopted by Hsu, Kakade, and
Zhang~\citep{hsu2012spectral}, which uses the definition of the
observable operators~\citep{jaeger2000observable}, $\mathbb{A}_x
  = \mathbb{T}\diag(\mathbb{O}[x,\cdot])$, explicitly. This allows them to
control the incoming and outgoing products in a decomposition analogous to the one above, but in the $\ell_1$ norm.
Their decomposition involves several terms, but to convey the main idea, observe that we can bound
\begin{align*}
  \sum_{x_1,\ldots,x_{t+1}} \| \rbr{\widehat{\mathbb{A}}_{x_{t+1}} - \mathbb{A}_{x_{t+1}}} \mathbb{A}_{x_t}\ldots,\mathbb{A}_{x_1} \|_1 \lesssim O\varepsilon\cdot\sum_{x_1,\ldots,x_t} \|\mathbb{A}_{x_t}\ldots,\mathbb{A}_{x_1} \|_1 \leq O\varepsilon.
\end{align*}
The idea is that each term in the final sum can be seen as a joint
probability of the history $x_{1},\ldots,x_t$ and the hidden state
$s_{t+1}$, so we can sum over all histories with no error
amplification. Unfortunately, there is no hidden state in the more
general setting (and for the rank deficient case, the observable operators can not be learned accurately as discussed in \Cref{sec:hard-back}), so we cannot appeal to an argument of this form.
Indeed, our main technical contribution is a new perturbation
analysis that relies on no structural assumptions.

At a more technical level, the issue with both of these arguments is
that passing to any norm, seems to be too coarse to adequately control
the error amplification. Instead, our argument carefully tracks the
error in the space of the coefficients.
Precisely, given
estimates $\widehat{A}_{o,t}$ that satisfy
$\|\widehat{A}_{o,t} - A_{o,t}\|_2 \leq \varepsilon$,
we can show, via an inductive argument, that for any $x_1,\ldots,x_t$
\begin{align*}
  (\widehat{A}_{x_t}\ldots\widehat{A}_{x_1} - A_{x_t}\ldots A_{x_1}) = \sum_{h \in H_t} \beta(h) \alpha_h + \sum_{v \in V_t^\perp} v \gamma_v,
\end{align*}
where $V_t^\perp$ is an orthonormal basis for the kernel of
$\Pr[F_t \mid B_t]$ and $\alpha_h,\gamma_v$ are scalars. Moreover, the TV distance between $\Pr[\cdot]$ and $\widehat{\Pr}[\cdot]$
is exactly equal to the sum of these scalars over all sequences
$x_1,\ldots,x_T$. Even though there could be exponentially many terms in this sum, we show that this sum is small via an inductive argument.
This makes up the most technical component of our proof, and we give a more detailed overview in \Cref{sec:main} with the formal proofs in \Cref{app:sec:perturb}.

\subsection{Estimating operators}\label{sec:intro:estimate}
We next discuss estimating the operators
$\{A_{o,t}\}_{o \in \obs, t \leq T}$ using the conditional sampling
oracle. A natural idea is to use samples to estimate both sides of the
system in~\Cref{eq:operator_full} and solve the noisy version via
linear regression.  Unfortunately, this system may have exponentially
small (in $T-t$) singular values, making it highly sensitive to
perturbation. There is also a cosmetic issue when working with
$\Pr[F_{t+1} | B_{t+1}]$, namely this matrix is exponentially
large.

To address these challenges, we introduce a particular preconditioner
that stabilizes the system.  Specifically, we instead estimate and
solve
\begin{align}
  \Pr[F_{t+1} | B_{t+1}]^\top D_{t+1}^{-1} \Pr[F_{t+1} | B_{t+1}] A_{o,t} = \Pr[F_{t+1} | B_{t+1}]^\top D_{t+1}^{-1} \Pr[o F_{t+1} | B_t] \, , 
\end{align}
where $D_{t+1}$ is a diagonal matrix with entries $d_{t+1}(f)
  := \frac{1}{|B_{t+1}|}\sum_{b \in B_{t+1}} \Pr[f | b]$ on the
diagonal.\footnote{This choice of $D_{t+1}$ ensures there is no division-by-zero issue, see \Cref{rem:zero}.}  The benefit of this preconditioner is that the new
matrices are of size $|B_{t+1}|\times|B_{t+1}|$ rather than
exponentially large, and yet they can still be estimated efficiently using the
conditional sampling oracle. To see why the latter holds, observe that
the $(i,j)^{\mathrm{th}}$ entry of the matrix on the LHS is
\begin{align*}
  \sbr{\Pr[F_{t+1} | B_{t+1}]^\top D_{t+1}^{-1}\Pr[F_{t+1} | B_{t+1}]}_{i,j} = \sum_{f \in F_{t+1}} d_{t+1}(f) \sbr{\frac{\Pr[f | b_i]\Pr[f |
        b_j]}{d_{t+1}(f)^2} },
\end{align*}
where $B_{t+1} = \{b_1,b_2,\ldots,\}$. Intuitively, we can estimate
this entry by sampling futures $f$ from $\Pr[\cdot | b]$ to
approximate any term in the sum and sampling futures from $d_{t+1}(\cdot)$
to approximate the sum itself. While this is true, there is one
technical issue to overcome: to estimate the ratio to
additive accuracy, we must estimate the individual probabilities
$\Pr[f \mid b_i]$, $\Pr[f \mid b_j]$ and $d_{t+1}(f)$ to relative
accuracy. We can obtain $(1 \pm \zeta)$ relative error estimates using
conditional samples as long as the one-step probabilities are at least
$\Omega(\zeta/T)$, but this is challenging when even a single one-step
probability is small.  To address this issue, we show that such
futures actually contribute very little to the overall sum, and we
design a test to safely ignore them. See \Cref{app:sub-1} for details.


While the ability to estimate the entries is clearly important, the
hope with preconditioning is that it dramatically amplifies the
singular values of the matrix on the left hand side. In particular, we
want that the matrix $\Pr[F_{t+1} | B_{t+1}]^\top
  D_{t+1}^{-1}\Pr[F_{t+1} | B_{t+1}]$ has large (non-zero) singular values,
as this will allow us to estimate the operators $A_{o,t}$ in the
$\ell_2$ norm. Our choice of preconditioner does achieve this in the
important example of parity with noise: we can show that $\Pr[F_{t+1}
    | B_{t+1}]$ has exponentially small (in $T-t$) singular values for
every choice of $B_{t+1}$, while there exists a basis $B_{t+1}$ for
which the non-zero singular values of the preconditioned matrix are $\Omega(1)$
(see~\Cref{sec:examples}). Unfortunately, in general, a basis which
ensures the preconditioner has large singular values might not exist, and
we address this by introducing the notion of fidelity.

\begin{restatable}[Fidelity]{definition}{deffidelity}\label{def:fidelity}
  We say that distribution $\Pr[\cdot]$ has fidelity $\Delta^*$ if there exists some bases $\{B_t\}_{t \in [T]}$, such that $\max_t |B_t| \leq 1/\Delta^*$ and
  \begin{align*}
    \forall t \in [T]: & ~~ \sigma_+\rbr{ S_t^{\frac{1}{2}}\Pr[F_t | H_t]^\top D_{t}^{-1} \Pr[F_t | H_t] S_t^{\frac{1}{2}}} \geq \Delta^*
  \end{align*}
  where $\sigma_+(M)$ denotes the magnitude of the smallest non-zero
  eigenvalue of $M$, $D_t$ is a diagonal matrix of size
  $|F_t|\times|F_t|$ with entries $d_t(f)
    := \frac{1}{|B_t|}\sum_{b \in B_t}\Pr[f | b]$, and $S_t$ is a
  diagonal matrix of size $|H_t|\times |H_t|$ with entries $s_t(h)
    := \Pr[h]$.
\end{restatable}

Importantly, we only assume the existence of bases with this
property, not that it is given to us or otherwise known in advance.
Note that, although the matrix with large  eigenvalues according to
the fidelity definition is not the same as the preconditioned matrix
we care about for learning operators, nevertheless when the
distribution has high fidelity (i.e., $\Delta^*$ is large), we can find
a basis for which $\Pr[F_{t+1} | B_{t+1}]^\top
  D_{t+1}^{-1}\Pr[F_{t+1} | B_{t+1}]$ has large eigenvalues.
This, combined with our approach for estimating entries of the
preconditioned matrix, allow us to learn operators
$A_{o,t}$ in the $\ell_2$ norm. We provide details in \Cref{app:sub:2}.

\begin{remark}
  Although our approach seems to require large fidelity, the parity with noise
  example suggests that this definition of fidelity, which can lead
  to a favorable preconditioned system, is more appropriate than
  directly assuming $\Pr[F_{t+1} | B_{t+1}]$ has large singular
  values. Indeed, we can also show that fidelity captures all previously
  studied positive results for learning HMMs. We also believe our approach can be extended to learn HMMs with small fidelity as described in \Cref{sec:open}.
\end{remark}



\subsection{Finding the basis} \label{sec:intro:base}
The only remaining challenge is to find the bases $\{B_t\}_{t \in
  [T]}$. Recall that, when considering the conditional sampling oracle,
we want bases for which the preconditioned matrices have large eigenvalues.
It turns out that when the distribution has high
fidelity a random sample of polynomially many histories will form a
basis with this property with high probability. Given that the other aspects of our
analysis seem to require high fidelity, this random sampling approach
thus suffices to prove~\Cref{thm:main}.

On the other hand, for low fidelity distributions, random sampling
will fail to cover the directions with small singular value, and so
basis finding becomes an intriguing aspect of learning with the
conditional sampling oracle. Basis finding is also the final issue to
address for~\Cref{thm:exact}, using the exact oracle. In both cases,
we provide adaptations of Angluin's $L^*$ algorithm that finds
bases for any low rank distribution. We defer discussion of the
conditional sampling version to~\Cref{sec:approx-ellip} and hope that it
serves as a starting point toward resolving~\Cref{open}.

\paragraph{Adapting $L^*$ for basis finding with the exact oracle.}
We close this section by explaining how to find a basis when provided
with the exact probability oracle. As a first observation, note that
we need not construct the entire system in~\Cref{eq:operator_full} to
identify operators $A_{o,t}$. It suffices to find a set of futures
$\Lambda_t \subset F_t$ such that $\Pr[\Lambda_t\mid H_t]$ spans the row space of
$\Pr[F_t\mid H_t]$. In other words, we just need $B_t$ and $\Lambda_t$
for which $\Pr[\Lambda_t\mid B_t]$ has the same rank as $\Pr[F_t \mid
    H_t]$.

The difficulty is that there is no universal choice of $B_t,\Lambda_t$
for general low rank distributions, and finding these sets poses a
challenge search problem in an exponentially
large space. We address this challenge using the
exact probability oracle and an adaptation of Angluin's $L^*$
algorithm for learning DFAs. The basic idea is as follows: given sets
$B_t,\Lambda_t$ whose submatrix is not of the required rank, we can still solve the
underdetermined system
\begin{align*}
  \Pr[\Lambda_t| B_t]A_{o,t} = \Pr[o\Lambda_t|B_t]
\end{align*}
and obtain an estimate $\widehat{\Pr}[\cdot]$ via~\Cref{eq:estimator}.
Then, we can sample sequences $x_1,\ldots,x_t \sim \Pr[\cdot]$ and
check if our estimate makes the correct predictions
on these sequences. In particular, we check
\begin{align*}
  \widehat{\Pr}[x_1,\ldots,x_t,\Lambda_t] \overset{?}{=} \Pr[x_1,\ldots,x_t,\Lambda_t].
\end{align*}
If the predictions are accurate (i.e., these equalities hold) for each $t$ and for polynomially
many random sequences, then we can show that $\widehat{\Pr}[\cdot]$
is close $\Pr[\cdot]$ in total variation distance.

On the other hand, if these equalities do not hold for some sample
$x_1,\ldots,x_t$, then we can use it as a counterexample to
improve our basis. We provide all the details
in \Cref{sec:toy}.

\section{Learning with conditional probabilities (Theorem 1)}
\label{sec:toy}
In this section we prove \Cref{thm:exact}.

\thmone*

We first introduce some notation, which differs from \Cref{sec:tech} slightly. We define $H_t := \obs^{t}$ to be the set of histories of length $t$. Similarly, we define $F_t := \obs^{\leq T-t}$ to be the set of futures of length $\leq T-t$, coinciding with our rank definition. Notice that unlike in \Cref{sec:tech}, we take $F_t$ to be all futures of length \emph{up to} $T-t$, so that one may append elements from the futures $F_t$ to elements from the histories $H_t$ to obtain a valid observation sequence of length at most $T$. To simplify the technical notation, let $\varphi$ be the empty string and define probabilities associated to empty string as: $\Pr[x_1 \ldots x_T | \varphi] = \Pr[x_1 \ldots x_T]$ and $\Pr[\varphi | x_1 \ldots x_T] = 1$ for any $T$-length sequence $x_1,\ldots, x_T$.

We now formally define the notion of bases for distribution $\Pr[\cdot]$.
\begin{restatable}[Basis]{definition}{defbasis}
    \label{def:bases}
    Let $\Pr[\cdot]$ be any distribution over observation sequences of length $T$. 
    A set $\{B_t\}_{t \in [T]}$, where each $B_t \subset H_t$, forms \emph{bases} for $\Pr[\cdot]$, if for each $t \in [T]$ and all $x \in \obs^t$, there exists coefficients $\beta(x)$ such that:
      \[
          \Pr[F_t|x]  = \Pr[F_t|B_t] \beta(x)\, .\]
  We call each $B_t$ a \emph{basis} for $\Pr[\cdot]$ at sequence length $t$.
  \end{restatable}

In other words, a set $B_t \subset H_t$ forms a basis for distribution $\Pr[\cdot]$ if the column vectors $\Pr[F_t | B_t]$ span the column space of $\Pr[F_t | H_t]$.
For now, when choosing $B_t$, we impose no constraint on the size of these coefficients, and we also do not require the columns $\Pr[F_t \mid B_t]$ to be linearly independent. 
The low rank property of $\Pr[\cdot]$ directly implies that for each $t$, there exists a basis $B_t$ with $|B_t| \leq r$.
However, as discussed in \Cref{sec:eff-rep}, there are exponentially many histories in $H_t$, so even if we had such a small basis $B_t$, simply learning the coefficients for each history will not suffice for an efficient algorithm. 
We address this issue with the following structural result:
because of the circulant structure of the conditional probability matrix, we can generate all the coefficients using $O T$ matrices each of size at most $r \times r$.

\begin{restatable}[Existence of efficient representation]{proposition}{propone}
    \label{prop:toy:1}
    Let $B_0 = F_T = \{\varphi\}$ and $B_T = \{h\}$ for any observation sequence $h \in H_{T}$.\footnote{We set $B_T$ to be a singleton set for notational clarity, as otherwise we would have to pre-multiply our probability estimate with the all ones row vector. Note that any singleton set forms a basis because $\Pr[F_T | H_T]$ is the all ones matrix.} For $t \in \{1,\ldots,T-1\}$, let $B_t \subset H_t$ be any basis for distribution $\Pr[\cdot]$ at sequence length $t$. Then, the probability distribution $\Pr[\cdot]$ can be written as\footnote{Here by choice of basis $B_0$ and $B_T$, $A_{x_T, T-1} = \Pr[x_T | B_{T-1}]$ by definition and is therefore a row vector. Similarly, $A_{x_1, 0}$ is a solution of $
    \Pr[F_1 | B_1] A_{x_1, 0} = \Pr[x_1 F_1 | \varphi]
  $ and is therefore a column vector.}:
  \begin{equation*}
    \Pr[x_1 \ldots x_T] = A_{x_T, T-1} A_{x_{T-1}, T-2} \ldots A_{x_1, 0}
  \end{equation*} where matrices $A_{o, t}$ for every $o \in \obs$ and $t \in \{0,\ldots,T-1\}$ satisfy \begin{equation}\label{eq:help:prop-1}
    \Pr[F_{t+1}| B_{t+1}]A_{o, t} = \Pr[oF_{t+1}|B_t].   
  \end{equation} Moreover, this equation always has a solution.
\end{restatable}
\begin{proof}
  We first show there exists a solution $A_{o, t}$ for \Cref{eq:help:prop-1}. For basis $B_t = \{b_1, \ldots, b_n\}$ and $B_{t+1}$, we claim the following $A_{o,t}$ is a solution: \begin{equation}
      A_{o, t} =  \begin{bmatrix}
        \beta(b_1 o) & \beta(b_2 o) & \cdots & \beta(b_n o)
      \end{bmatrix} \begin{bmatrix}
        \Pr[o| b_1] & 0           & \cdots & 0           \\
        0           & \Pr[o| b_2] & \cdots & 0           \\
        \vdots      & \vdots      & \ddots & 0           \\
        0           & \cdots      & 0      & \Pr[o| b_n]
      \end{bmatrix}\label{eq:operator_closed}
    \end{equation} Here $\beta(x) $ and $\beta(xo)$ are the coefficients associated to history $x$ of length $t$ under $B_t$ and history $xo$ of length $t+1$ under $B_{t+1}$ respectively. Recall that, in particular, these coefficients are such that $\Pr[F_{t+1} \mid B_{t+1}]\beta(xo) = \Pr[F_{t+1} \mid xo]$. By definition of $A_{o,t}$, \begin{align}
      & \Pr[F_{t+1} | B_{t+1}] A_{o,t} \notag                                             \\
    = & \Pr[F_{t+1} | B_{t+1}] \begin{bmatrix}
                                        \beta(b_1 o) & \beta(b_2 o) & \cdots & \beta(b_n o)
                                    \end{bmatrix} \begin{bmatrix}
                                                      \Pr[o| b_1] & 0      & \cdots & 0           \\
                                                      \vdots      & \vdots & \ddots & 0           \\
                                                      0           & \cdots & 0      & \Pr[o| b_n]
                                                  \end{bmatrix}\notag  \\
    = & \Pr[F_{t+1} | B_{t} o] \begin{bmatrix}
                                       \Pr[o| b_1] & 0      & \cdots & 0           \\
                                       \vdots      & \vdots & \ddots & 0           \\
                                       0           & \cdots & 0      & \Pr[o| b_n]
                                   \end{bmatrix} \tag{by definition of $\beta(b_i o)$} \\
    = & \Pr[o F_{t+1} | B_{t}]. \tag{by Bayes rule}
\end{align} Since $oF_{t+1}$ is a subset of $F_t$, by repeatedly applying this equation, we get \[
  \Pr[F_T | B_T] A_{x_T, T-1} A_{x_{T-1}, T-2} \ldots A_{x_1, 0} = \Pr[x_T F_T \mid B_{T-1}] A_{x_{T-1},T-2}\ldots A_{x_1,0} = \Pr[x_1 x_2\ldots x_T F_T | B_0]
\]
Noting $\Pr[F_T | B_T] = 1$ and  $\Pr[x_1 x_2\ldots x_T F_T | B_0] = \Pr[x_1 x_2\ldots x_T \varphi | \varphi] = \Pr[x_1 x_2\ldots x_T]$ as $F_T = B_0 = \{\varphi\}$ completes the proof.
\end{proof} 

\begin{savenotes}
\begin{algorithm}[t!]
  \SetArgSty{textrm}
  \caption{\label{alg:exact}Learning low rank distributions using exact conditional probabilities.}
  \SetAlgoLined
  Set $B_0 = \Lambda_T = \{\varphi\}$.\\
  Set $B_t = \{0^t\}$ where $0^t$ is $(0, \ldots, 0)$ with $t$ zeroes for all $t \in \{1, \ldots, T\}$.\\
  Set $\Lambda_{t} = \{0\}$ or $\{1\}$ to ensure $\Pr[\Lambda_t | B_t] \neq 0$ for all $t \in \{0, \ldots, T-1\}$.
  \footnote{either $\Pr[0|B_t]$ or $\Pr[1 \mid B_t]$ must be nonzero.}\\ 
  \For{round $1,2, \ldots$}{
    Choose $\widehat A_{o,t}$ for each $o\in \obs$ and $t \in [T-1]$ to be any matrix that satisfies \begin{equation}\label{app:eq:1}
      \Pr[\Lambda_{t+1} | B_{t+1}] \widehat A_{o, t} = \Pr[o \Lambda_{t+1} | B_t]
    \end{equation} \label{app:line:1}\\
    Let $\overline \Pr$ be a function defined on observation sequence $(x_1 \ldots x_t)$ for any $t\in [T]$ as, \begin{equation}\label{app:eq:1star}
      \overline \Pr[x_1, \ldots, x_t, \Lambda_{t}] = \Pr[\Lambda_t| B_t] \widehat A_{x_t, t-1} \ldots \widehat A_{x_1, 0}
    \end{equation}\\
    Sample $n$ sequences $(x_1,\ldots x_t)$ for each length $t \in [T]$ and check if any one of these $n T$ sequences is a counterexample, i.e., it satisfies  \begin{equation*}
      \overline \Pr[x_1, \ldots, x_t, \Lambda_t] \neq \Pr[x_1, \ldots, x_t, \Lambda_t]
    \end{equation*} \label{app:line:2}\\
    \If{we find such a counterexample $(x_1, \ldots, x_t)$}{
      Use \Cref{app:prop:1} to find a time step $\tau \in [t]$, a new test future $\lambda' \in F_\tau$, and a new representative history $b'\in H_\tau$. Update $\Lambda_\tau := \Lambda_\tau \cup \{\lambda'\}$ and $B_\tau := B_\tau \cup \{b'\}$.
    }
    \Else{
      return $\{\widehat A_{o,t}\}_{o \in \obs, t \in [T-1]}$
    }
  }
\end{algorithm}
\end{savenotes}

\subsection{Algorithm}
We now present our algorithm (\Cref{alg:exact}). The user furnishes $\varepsilon$, the accuracy with which the distribution is to be learned; and $\delta$, a confidence parameter. The parameter $n$ depends on the input and is detailed in the proof of \Cref{thm:exact}.

As discussed in \Cref{sec:tech}, the algorithm iteratively builds a set of histories $B_t \subset H_t$ and a set of futures $\Lambda_t \subset F_t$ (for each $t$), to span the column/row space of $\Pr[F_t \mid H_t]$, respectively. Via \Cref{prop:toy:1}, if we can find such sets, they would provide an efficient representation of the distribution. We refer to $B_t$ and $\Lambda_t$ as \emph{representative} histories and \emph{test} futures, respectively, and as we grow these sets, we maintain the invariant that the matrix $\Pr[\Lambda_t \mid B_t]$ is square and invertible.

We start with $B_t,\Lambda_t$ of size $1$. Then, we repeat the
following: motivated by the evolving equation
in~\Cref{prop:toy:1}, we use~\Cref{app:eq:1} to compute estimates
$\widehat{A}_{o,t}$ using our current representative histories and test
futures. This may be an under-determined linear system,
but \Cref{prop:toy:1} guarantees that it has a solution, and we take
$\widehat{A}_{o,t}$ to be any such solution. We use these operators to
define our estimate for the distribution, given
in~\Cref{app:eq:1star}, via iterated multiplication of the
operators. Then, we sample several sequences from the distribution and
check if any of them certify that our estimate is incorrect, i.e.,
serve as a counterexample. If we do find a counterexample, then the
algorithm finds a time step $\tau$, a new history $b' \in H_\tau$ and
a new future $\lambda' \in F_\tau$ that increases the rank of
$\Pr[F_\tau \mid B_\tau]$ (this step is described in \Cref{app:prop:1}
below). This can only happen $rT$ times if the distribution has rank
$r$. On the other hand, if we do not find a counterexample, then we
simply output our current estimate.

\subsection{Analysis}
We first show how to use a counterexample to improve our set of representative histories and test futures.
  
\begin{restatable}[Finding representative histories and test futures]{proposition}{proptwo}\label{app:prop:1}
    If $x_1\ldots x_t$ is a counterexample, that is, it satisfies the following:\begin{equation} \label{app:prop1:eq1}
      \overline \Pr[x_1, \ldots, x_t, \Lambda_{t}] \neq \Pr[x_1, \ldots, x_t, \Lambda_{t}] 
    \end{equation}
   then we can find a new test future $\lambda' \in F_\tau$ and representative history $b' \in H_\tau$ for $\tau \in [t]$ in at most $\poly(r, T)$ time such that $\rank(\Pr[\Lambda_\tau \cup \{\lambda'\} | B_\tau \cup \{b'\}]) = \rank(\Pr[\Lambda_\tau | B_\tau]) + 1$.
\end{restatable}
\begin{proof}
  For clarity, in the poof, we abuse notation and do not explicitly mention the sequence length when writing the operator $A_{o,t}$, i.e., we use $A_{x_t}$ instead of $A_{x_t, t-1}$. First, we find a time $\tau \in [t]$ where the following equations hold:\begin{align*}
      \Pr[x_1 \ldots x_\tau \Lambda_\tau] &= \Pr[\Lambda_\tau | B_\tau] \widehat A_{x_\tau} \ldots \widehat A_{x_1}\\
      \Pr[x_1 \ldots x_\tau x_{\tau + 1} \Lambda_{\tau + 1}] &\neq \Pr[\Lambda_{\tau + 1} | B_{\tau + 1}] \widehat  A_{x_{\tau + 1}} \widehat A_{x_\tau} \ldots \widehat  A_{x_1}
  \end{align*}
  Such a $\tau$ must exist because (a) the first equation is true for $\tau = 0$ by definition, and (b) the second equation is true for $\tau = t-1$ because of the counterexample property (\Cref{app:prop1:eq1}).
  Now, we can simplify the equations above by substituting the vector $\mathbf{v} := (\Pr[x_1 \ldots x_\tau])^{-1} \widehat A_{x_\tau} \ldots \widehat A_{x_1}$ which gives \begin{align}
    \label{app:prop:eq:2}
    \Pr[\Lambda_\tau | x_1 \ldots x_\tau] &= \Pr[\Lambda_\tau | B_\tau] \mathbf{v}\\
      \Pr[x_{\tau + 1} \Lambda_{\tau + 1} | x_1 \ldots x_\tau] &\neq \Pr[\Lambda_{\tau + 1} | B_{\tau + 1}] \widehat A_{x_{\tau + 1}} \mathbf{v} = \Pr[x_{\tau + 1}\Lambda_{\tau + 1} | B_{\tau}] \mathbf{v}, \label{app:prop:eq:3}
  \end{align} where the last step holds by definition of $\widehat A_{x_{\tau + 1}}$ (\Cref{app:eq:1}). Let $x_{\tau + 1}\lambda_{\tau + 1}$ index the row of \Cref{app:prop:eq:3} where equality does not hold. Define $\lambda' = x_{\tau + 1} \lambda_{\tau + 1}$ and $b' = x_1 \ldots x_\tau$. We show that the equations above imply that the row vector $\Pr[\lambda' | B'_{\tau}] := \Pr[x_{\tau + 1} \lambda_{\tau + 1} | B'_{\tau}]$ is linearly independent of the rows of $\Pr[\Lambda_{\tau} | B'_{\tau}]$. This is enough to prove our claim that $\rank(\Pr[\Lambda'_\tau | B'_{\tau}]) = \rank(\Pr[\Lambda_\tau | B_{\tau}]) + 1$. 
  
  We establish linear independence by contradiction. Assume that $\Pr[\lambda' \mid B'_{\tau}]$ is in the span of the rows of $\Pr[\Lambda_\tau \mid B'_\tau]$. Then, there exists a vector $\textbf{w}$ such that: \begin{equation}\label{app:prop:eq:4}
    \Pr[x_{\tau + 1} \lambda_{\tau + 1} | B'_{\tau}] = \textbf{w}^\top \Pr[\Lambda_{\tau} | B'_{\tau}]\, .
  \end{equation} Then, we reach a contradiction as \begin{align*}
    \Pr[x_{\tau + 1}\lambda_{\tau + 1} | x_1 \ldots x_{\tau}] &= \textbf{w}^\top \Pr[\Lambda_{\tau} | x_1 \ldots x_{\tau}] \\
      &=  \textbf{w}^\top \Pr[\Lambda_\tau | B_\tau] \textbf{v}\\
      &=  \Pr[x_{\tau + 1} \lambda_{\tau + 1} | B_{\tau}] \textbf{v}\\
      &\neq \Pr[x_{\tau + 1} \lambda_{\tau + 1} | x_1 \ldots x_\tau]
  \end{align*}
 where the first and third equality follows from linear dependence (\Cref{app:prop:eq:4}), the second equality follows from \Cref{app:prop:eq:2}, and the last inequality follows from \Cref{app:prop:eq:3}.
\end{proof}

Finally, we need a technical lemma which allows us to estimate the TV distance using conditional samples. This lemma implies that if our algorithm does not find a violation, then with high probability our estimate is close to the true distribution in TV distance.

\begin{restatable}[Substitute for TV oracle]{proposition}{propsub}
  \label{lemma:condcheck}
  Let $\Pr[\cdot]$ and $\widehat \Pr[\cdot]$ be two probability distributions over observation sequences of length $T$. Suppose that for all $t\in \{0, \ldots, T\}$ and observations $o\in \obs$ \[
      \EE_{x_1, \ldots, x_t \sim \Pr[\cdot]} \sbr{\abr{\widehat \Pr[o | x_1, \ldots, x_t] -  \Pr[o | x_1, \ldots, x_t]}} \leq  \varepsilon\, .
  \] Then \[
      TV(\Pr, \widehat \Pr) = \frac{1}{2}\sum_{x_1, \ldots, x_T} |(\Pr[x_{1:T}] - \widehat \Pr[x_{1:T}])| \leq \frac{(T+1)|O|\varepsilon}{2}
  \]
\end{restatable}

Since, $\overline \Pr[\cdot]$ might not be a probability distribution, we need to apply this proposition to a probability distribution $\widehat \Pr[\cdot]$ that is close to $\overline \Pr[\cdot]$, which can be obtained by a simple construction. These details and the proof of \Cref{lemma:condcheck} are relatively straightforward and deferred to~\Cref{app:sec:exact}. 


\section{Learning with conditional samples (Theorem 2)}
\label{sec:main}
In this section, we prove \Cref{thm:main}

\thmtwo*

Throughout this section, we use the same notation as \Cref{sec:tech}, the set of futures $F_t:= \obs^{T-t}$. 

\begin{algorithm}[t!]
  \SetArgSty{textrm}
  \caption{Learning low rank distributions using conditional samples.}
  \label{alg:cond}
  \SetAlgoLined
  \For{sequence length $t = 0,1,2, \ldots, T$}{
      Build set $B_t = \{b_1, \ldots, b_n\}$ of $n$ observation sequences of length $t$ using \Cref{prop:1}.\\
      Build empirical estimates $\widehat q(b o)$ and $\widehat \Sigma_{B_t}$ (defined in \Cref{eq:q_def} and \Cref{eq:sigma_def}) for each history $b \in B_t$, observations $o \in \obs$ using \Cref{cor:estimatesigma} with $m$ conditional samples.\\
      Compute SVD of $\widehat\Sigma_{B_t}$.\\
      Let $\widehat V_{t}$ be the matrix of eigenvectors corresponding to eigenvalues $> \Delta/2$.\\
      Compute coefficients $\widehat \beta(b'_i o)$ for each observation $o \in \obs$ and sequence $b'_i \in B_{t-1}$ by solving:
  \[
    \widehat \beta(b'_i o) = \argmin_z \norm{\widehat \Sigma_{B_{t}} z - \widehat q(b'_i o)}^2_2 + \lambda \norm{z}_2^2.
  \]\\
  Compute model parameters $\widehat A_{o,t-1}$ for each observation $o \in \obs$:\begin{equation}\label{eq:new:alg}
    \widehat A_{o,t-1} = \widehat V_{t} \widehat V_{t}^\top \begin{bmatrix}
      \widehat \beta(b'_1 o) & \widehat\beta(b'_2 o) & \cdots & \widehat\beta(b'_n o)
    \end{bmatrix}  \begin{bmatrix}
      \widehat\Pr[o| b'_1] & 0           & \cdots & 0           \\
      0           & \widehat\Pr[o| b'_2] & \cdots & 0           \\
      \vdots      & \vdots      & \ddots & 0           \\
      0           & \cdots      & 0      & \widehat\Pr[o| b'_n]
    \end{bmatrix} \widehat V_{t-1} \widehat V_{t-1}^\top .
  \end{equation}
  }
  Return model parameters $\{\widehat A_{o,t}\}$.
\end{algorithm}
\subsection{Algorithm}
Algorithm pseudocode is displayed in~\Cref{alg:cond}. The user furnishes $\varepsilon$, the accuracy with which the distribution is to be learned; $\delta$, a confidence parameter; $\Delta^*$, the fidelity of the distribution and $r$, the rank of the distribution. The parameters $\Delta, \lambda, n$ and $m$ are detailed in the proof of \Cref{thm:main} in \Cref{sec:main-result}. 

As with the previous algorithm,~\Cref{alg:cond} relies on the
efficient representation provided by~\Cref{prop:toy:1}. First, the
algorithm finds basis histories $B_t$ for each $t \in [T]$. As
discussed in \Cref{sec:tech}, under the fidelity assumption, this is
not particularly challenging and can be done by sampling from the
distribution. The remaining steps in the algorithm constitute a
specialized technique for estimating the operators $A_{o,t-1}$
specified in~\Cref{prop:toy:1}.

Our estimate $\widehat{A}_{o,t-1}$ is based on the formula for
$A_{o,t-1}$ given in~\Cref{eq:operator_closed} and involves three
components: (a) projection onto (an estimate of) the row space of $\Pr[F_t \mid B_t]$, (b) estimates of coefficients
$\beta(b o)$ and (c) estimates of probabilities $\Pr[o \mid b]$, where
the latter two are for $b_i \in B_{t-1}$. Item (c) is straightforward
using conditional samples. For item (a), we define the
``preconditioned matrix''
\begin{align}
    \Sigma_{B_t} := \Pr[F_t \mid B_t]^\top D^{-1}_t \Pr[F_t \mid B_t],  \label{eq:sigma_def}
\end{align}
where $D_t$ is a $|F_t|\times |F_t|$ diagonal matrix with $d_t(f)
= \frac{1}{|B_t|}\sum_{b \in B_t}\Pr[f \mid b]$ on the diagonal. We show in \Cref{app:sub-1}, how this
matrix can be estimated using conditional samples. We project onto the
principal subspace of the estimated matrix, i.e., onto the span of the
eigenvectors with eigenvalue larger than $\Delta/2$. These projections
help with error propagation, as it eliminates errors that leave the
principal subspace.

For item (b), we estimate the coefficients $\beta(bo)$ for $b \in
B_{t-1}$ via linear regression. Using our preconditioner and the
definition of the coefficients, we can see that the coefficients
satisfy:
\begin{align}
    q(bo) = \Sigma_{B_t}\beta(bo) \qquad \text{where} \qquad q(bo) := \Pr[F_t \mid B_t]^\top D_t^{-1}\Pr[F_t \mid bo] \label{eq:q_def}
\end{align}
As with $\Sigma_{B_t}$, $q(bo)$ can also be estimated using
conditional samples (via the approach in \Cref{app:sub-1}). Moreover, our
basis $B_t$ will ensure that $\|\beta(bo)\|_2$ is bounded by a
universal constant, so we can use ridge regression to find estimates
$\widehat{\beta}(bo)$. Then we can plug these into~\Cref{eq:new:alg}
to obtain estimates $\widehat{A}_{o,t-1}$.  We return these matrices
as the representation of our estimated distribution.

  \subsection{Analysis} In the previous setting, when we had access to
exact conditional probability oracle, the main challenge was finding the bases. In contrast, now
that we can only obtain samples, even if we know the bases, we can only learn operators $A_{o,t}$ approximately. As
discussed in~\Cref{sec:tech}, controlling estimation errors will
require the notion of \emph{robust bases}, which we define next.

\begin{definition}[Robust bases]\label{def:robust-basis}
Bases $\{B_t\}_{t \in [T]}$ for distribution $\Pr[\cdot]$ are \emph{$\Delta$-robust} if for every $t \in [T]$:
       \[
        \sigma_+ \rbr{\Pr[F_t | B_t]^\top D_{t}^{-1} \Pr[F_t | B_t]} \geq \Delta
      \] where $\sigma_+(M)$ denotes the minimum non-zero eigenvalue of $M$ and $D_{t}$ is a diagonal matrix of size $|F_t| \times |F_t|$ with entries $d_{t}(f) := \frac{1}{|B_t|}\sum_{b \in B_{t}} \Pr[f | b]$ on the diagonal.
\end{definition}

A priori, it is unclear if such bases exists for arbitrary low rank distributions. Moreover, even if robust bases exist, how do we find them? Our first lemma show how to find robust bases for high fidelity distributions (\Cref{def:fidelity}).

  \begin{restatable}[Finding robust bases]{lemma}{lemmabasis}\label{prop:1}
   Assume distribution $\Pr[\cdot]$ has rank $r$ and fidelity $\Delta^*$. Pick $0< \delta< 1$. Let $n = O(\Delta^{* - 8} \log (r/\delta T) )$ and $\Delta = \Omega(\Delta^{* - 11/2}\log (r/\delta T) )$. For each $t \in [T]$, let $S_t$ be a random sample of size $n$ of observation sequences of length $t$ from distribution $\Pr[\cdot]$. Then, with probability $1-\delta$, $\{S_t\}_{t\in [T]}$ form $\Delta$-robust bases for $\Pr[\cdot]$.
\end{restatable}

We provide a proof in \Cref{sec:find-robust-proof}. According to this lemma, a random sample from a high fidelity distribution forms a robust basis for each $t$. 
With access to a $\Delta$-robust basis $B_t$, we turn to
the issues of estimation and error analysis. First we study estimation
of the preconditioned quantities $q(bo)$ and $\Sigma_{B_t}$ used by
the algorithm. Note that all entries of these vectors and matrices are
of the following form, where $b^* \in B_t$ and $x$ is a history of length $t$:\[
  s(b^*, x) = \sum_{f \in F_t}  \frac{\Pr[f|b^*]\Pr[f|x]}{d(f)}\, .
\] We show such quantities can be estimated efficiently using conditional samples.

\begin{lemma}[Estimating preconditioned quantities]\label{prop:estimate-sum1-re}
  Let $\{B_t\}_{t \in [T]}$ be bases for distribution $\Pr[\cdot]$ where $\max_{0 \leq t \leq T} |B_t| \leq n$. Pick any $0 < \varepsilon, \delta < 1$. Fix $b^* \in B_t$ and $x \in H_t$. Then we can build estimate $\widehat s(b^*, x)$ in $\poly(n, |\obs|, T, 1/\varepsilon, \log(1/\delta))$ time such that with probability $1-\delta$, \[
      \abr{s(b^*, x) - \widehat s(b^*, x)} \leq \varepsilon\, .
  \]
  
\end{lemma}

We provide the estimation algorithm and a proof in \Cref{app:sub-1}. Using this lemma, we can estimate the operators $A_{o,t}$ via~\Cref{eq:new:alg}. The next lemma provides a precise characterization of the estimation error for these operators.
\begin{lemma}[Estimating operators]\label{lemma:cond:1}
  Assume the distribution $\Pr[\cdot]$ has rank $r$ and that $\{B_t\}_{t \in [T]}$ are $\Delta$-robust bases.
Pick $0 < \varepsilon, \delta < 1$. Then, 
we can learn approximations $\widehat A_{o, t}$ for all observations $o \in \obs$ and $t \in [T]$ in $\poly(r$, $|\obs|$, $T$, $1/\varepsilon$, $1/\Delta$, $\log(1/\delta))$ time such that with probability $1-\delta$, for any unit vector $v$ \begin{align*} 
    (\widehat A_{o, t} - A_{o, t}) v = \beta(B_{t+1}) \alpha(o,v) + V_{t+1}^\perp \alpha^\perp(o,v),
\end{align*} 
where $\beta(B_{t+1})$ is a matrix with columns $\beta(b)$ for $b \in B_{t+1}$,
$V_{t+1}^\perp$ is a matrix whose columns form an orthonormal basis for the kernel of 
$\Pr[F_{t+1} | B_{t+1}]^\top D_{t+1}^{-1} \Pr[F_{t+1} | B_{t+1}]$, and the vectors $\alpha(o,v)$ and $\alpha^\perp(o,v)$ are $\ell_1$ bounded, i.e., \[
   \max( \|\alpha(o,v)\|_1, \|\alpha^\perp(o,v)\|_1) \leq \varepsilon\, .
\]
\end{lemma}

We provide the proof in \Cref{app:sub:2}. 
%
%
As noted in \Cref{sec:tech}, the main technical challenge is in analyzing how the estimation error propagates to errors in induced distributions. 
%
%
Using this structured error, we can show how to bound the TV distance between the induced distributions. 
\begin{lemma}[Perturbation argument]\label{lemma:error-prop-1}
  Assume for each sequence length $t \in [T]$ and observation $o \in \obs$, we have an operator $\widehat A_{o,t}$ which is close to $A_{o,t}$ as defined above in \Cref{lemma:cond:1}. Let $\widehat \Pr[\cdot]$ be a function over observation sequences of length $T$ given by \[
\widehat \Pr[x_1\ldots x_T] = \widehat A_{x_T, T-1} \ldots \widehat A_{x_1, 0}  
\] Then, the functions $\Pr[\cdot]$ and $\widehat \Pr[\cdot]$ are close in TV distance:\[
  \mathrm{TV}(\Pr, \widehat \Pr) \leq 2 |\obs| T \varepsilon
\]
\end{lemma}

This makes up the most technical component of our proof, and we give the formal proofs in \Cref{app:sec:perturb}. Together with previous lemmas, this proves our main theorem, \Cref{thm:main}. We provide the formal proof in \Cref{sec:main-result}.

\section{Discussion}
\label{sec:open}
In this paper we show how interactive access to hidden Markov models
(and more generally low rank distributions) can circumvent
computational barriers to efficient learning. In particular, we show
that all low rank distributions with a certain fidelity property can
be efficiently learned assuming access to a conditional sampling
oracle. In~\Cref{sec:examples}, we show that fidelity captures the
assumptions considered in prior work on (non-interactive) learning of HMMs, specifically:
\begin{itemize}
\item Parity with noise admits bases $B_t$ each of cardinality $2$ with fidelity $(1-2\alpha^2)/2$, where $\alpha$ is the noise parameter.
\item Full rank HMMs, where $\mathbb{T}$ and $\mathbb{O}$ are full column rank, admit bases of size $O$ with fidelity bounded by the minimum singular value of the second moment matrix $\Pr[\textbf{x}_2=\cdot,\textbf{x}_1=\cdot]$. This parameter also appears polynomially in the analysis of~\citet{hsu2012spectral}.
\item The overcomplete setting of~\citet{sharan2017learning}, where sequences of length $\log S$ are used for estimation, admits bases of size $S$ with fidelity $1/\poly(S)$, matching their parameters.
\end{itemize}

Despite this, the reliance on the fidelity parameter is the main
limitation of our results. We believe this dependence is not necessary,
which leads to the main open problem,~\Cref{open}. We close the paper
with some final remarks regarding this open problem.

As we have mentioned previously, although fidelity greatly simplifies
the basis finding aspect of our algorithm, it is not necessary for
this part and refer the reader to~\Cref{sec:approx-ellip} where we give a
general algorithm for basis finding. Indeed the only place where
fidelity is required is in our error propagation analysis, where our
techniques require that operators $\widehat{A}_{o,t}$ are estimated
in $\ell_2$ norm. In the general case, we will only be able to learn
operators in the directions for which the preconditioned matrix has
large eigenvalues, and ideally we should be able to ignore the directions with small
eigenvalues. This strategy would work if we can show that ignoring the
small directions preserves the low rank property, which is the linear-algebraic analog of 
approximating an HMM by one with fewer states. Unfortunately, we do
not know if the latter holds, and we believe this is the key challenge
to resolving~\Cref{open}. We look forward to further progress on this problem.

\section*{Acknowledgement}
SK acknowledges funding from the Office of Naval Research under award N00014-22-1-2377 and the National Science Foundation Grant under award \#CCF-2212841.
This work has been made possible in part by a gift from the Chan Zuckerberg Initiative Foundation to establish the Kempner Institute for the Study of Natural and Artificial Intelligence.
We thank Weihao Kong, Vatsal Sharan, Geelon So and Shachar Lovett for insightful discussions.

\printbibliography
\newpage
\appendix

\section{Proofs for \texorpdfstring{\Cref{sec:toy}}{Exact Setting}}
\label{app:sec:exact}

\thmone*
\begin{proof}
    First, the total number of rounds are at most $rT$. This is because by \Cref{app:prop:1}, we increase $\rank(\Pr[\Lambda_\tau|B_\tau])$ by $1$ in every round for some $\tau \in [T]$, and $\rank(\Pr[\Lambda_\tau|B_\tau])$ can be at most $r$ by our low rank assumption. So we only need to show that when the algorithm ends, we have found a good estimate. Note that this happens when we can not find a counterexample in \Cref{app:line:2}. 
    
    Let $z_t$ be the initial bit choice of $\Lambda_t$ i.e. $z_t = 0$ if initial choice of $\Lambda_t = \{0\}$ and $z_t = 1$ if initial choice of $\Lambda_t = \{1\}$.
    By Hoeffding's inequality, for $n = O(\log(Tr/\delta)/\varepsilon^{2})$, we get with probability $1-\delta/Tr$ for all $t\in \{0, \ldots, T-1\}$: \begin{equation}\label{eq:f1}
        \Pr_{x_{1:t}}\sbr{\Pr[x_1 \ldots x_t, z_t] \neq \overline \Pr[x_1 \ldots x_t, z_t]} \leq \varepsilon
    \end{equation} as $\Lambda_t$ contains $z_t$ for all such $t$. Moreover, define a probability distribution $\overline \Pr$ over sequences of length up to $T$ using $\widehat \Pr$ as follows: for any $t\in \{0, \ldots, T-1\}$,
\begin{align*}
        \widehat \Pr[z_t | x_1 \ldots x_t] & = \Pi_{[0,1]}\sbr{\frac{\overline \Pr[x_1 \ldots x_t z_t]}{\widehat \Pr[x_1 \ldots x_t]}} \\
        \widehat \Pr[1 - z_t | x_1 \ldots x_t] & = 1 - \widehat \Pr[z_t | x_1 \ldots x_t]
    \end{align*} where $\Pi_{[0,1]}$ projects onto interval $[0,1]$. Note that for a sequence $(x_1\ldots x_T)$, if for all $t\in \{0, \ldots, T-1\}$, $\Pr[x_1\ldots x_t \Lambda_t] = \overline \Pr[x_1\ldots x_t \Lambda_t]$, then $\Pr[x_1\ldots x_T] = \widehat \Pr[x_1\ldots x_T]$. Therefore, together with \Cref{eq:f1}, we get for each $t\in \{0, \ldots, T-1\}$ and $o \in \obs$, \[
        \Pr_{x_{1:t}}\sbr{\Pr[o| x_{1:t}] \neq \widehat \Pr[o | x_{1:t}]} \leq 2 T \varepsilon
    \] which implies \[
        \EE_{x_{1:t}}\sbr{\abr{\Pr[o| x_{1:t}] - \widehat \Pr[o | x_{1:t}]}} \leq 2 T \varepsilon
    \] Using \Cref{lemma:condcheck}, we get for distribution $\widehat \Pr[\cdot]$ \[
        TV(\widehat \Pr, \Pr) \leq 2T(T+1)\varepsilon
    \] Re-substituting the value of $\varepsilon$, we get $TV(\widehat \Pr, \Pr) \leq \varepsilon$ using at most $O(rT^5\log(Tr/\delta)/\varepsilon^2)$ samples from $\Pr[\cdot]$
    and queries to the exact conditional probability oracle.
\end{proof}

\propsub*
\begin{proof}
    We prove this by induction. Assume \[
        \sum_{x_{1: t-1}} \abr{(\Pr[x_{1: t-1}] - \widehat \Pr[x_{1: t-1}])} \leq t|O|\varepsilon
    \] Then, \begin{align*}
         & \sum_{x_{1:t}} \abr{(\Pr[x_{1:t}] - \widehat \Pr[x_{1:t}])}                                                                                                                    \\
         & = \sum_{x_{1:t}} \abr{\Pr[x_{1: t-1}] \Pr[x_t| x_{1: t-1}] - \widehat \Pr[x_{1: t-1}] \widehat \Pr[x_t| x_{1: t-1}]}                                                                   \\
         & \leq \sum_{x_{1:t}} \Pr[x_{1: t-1}] \cdot \abr{\Pr[x_t| x_{1: t-1}] - \widehat \Pr[x_t| x_{1: t-1}]} + \sum_{x_{1:t}} \abr{(\Pr[x_{1: t-1}] - \widehat \Pr[x_{1: t-1}])} \cdot \widehat \Pr[x_t| x_{1: t-1}]
    \end{align*} We handle the two terms separately. The first term \begin{align*}
         & \sum_{x_{1:t}} \Pr[x_{1: t-1}] \cdot  \abr{\Pr[x_t| x_{1: t-1}] - \widehat \Pr[x_t| x_{1: t-1}]}                \\
         & = \sum_{x\in \obs}  \sum_{x_{1: t-1}} \Pr[x_{1: t-1}] \cdot \abr{\Pr[x| x_{1: t-1}] - \widehat \Pr[x| x_{1: t-1}]} \\
         & = \sum_{x\in \obs} \EE_{x_{1: t-1}} \sbr{\abr{\Pr[x| x_{1: t-1}] - \widehat \Pr[x| x_{1: t-1}]}}        \\
         & \leq |O|\varepsilon
    \end{align*} where the last step follows from our assumption. The second term \begin{align*}
         & \sum_{x_{1:t}} \abr{(\Pr[x_{1: t-1}] - \widehat \Pr[x_{1: t-1}])} \cdot \widehat \Pr[x_t| x_{1: t-1}]                 \\
         & = \sum_{x_{1: t-1}} \abr{\Pr[x_{1: t-1}] - \widehat \Pr[x_{1: t-1}]} \cdot \sum_{x \in \obs}  \widehat \Pr[x| x_{1: t-1}] \\
         & \leq \sum_{x_{1: t-1}} \abr{\Pr[x_{1: t-1}] - \widehat \Pr[x_{1: t-1}]}                                      \\
         & \leq t |O|\varepsilon
    \end{align*}
\end{proof}

\section{Proofs for \texorpdfstring{\Cref{sec:main}}{Noisy Setting}}
\label{app:sec:main}

In this section, we will always restrict attention to
basis $B_t$, for which the norm of the coefficient
$\norm{\beta(x)}_2 \leq c$ for some universal constant $c \leq 1$ and
every history $x \in H_t$. Note that by repeating elements in
the basis, we can always make the norm smaller than $1$. So, this norm
bound is without loss of generality. Finding such basis is more
involved, and we show how to find them later on.

\subsection{Finding robust basis}\label{sec:find-robust-proof}
We first show how to find a robust basis. Let us define some notation: the covariance matrix for a basis is defined as \begin{align*}
    \Sigma_{B_t}      & = \Pr[F_t|B_t]^\top D_t^{-1} \Pr[F_t|B_t]                  
\end{align*} where $D_t$ is a diagonal matrix of size $|F_t| \times |F_t|$ with entries $d_t(f) := \EE_{b \in B_t} \Pr[f | b]$ on the diagonal. We further define the inner covariance as \begin{align*}
    \bar \Sigma_{B_t} & = D_t^{-1/2} \Pr[F_t|B_t] \Pr[F_t|B_t]^\top D_t^{-1/2}
\end{align*} Note that the matrices $\Sigma_B$ and $\bar \Sigma_B$ share their non-zero eigenvalues (with extra eigenvalues all $0$).  As we will see, random sampling from a high fidelity distribution gives a robust basis. We show in \Cref{sec:approx-ellip} more efficient ways of building a basis.
\lemmabasis*
\begin{proof}
    Let $S_t$ be a random sample of size $n$ of observation sequences of length $t$ from distribution $\Pr[\cdot]$. Let $B_t^*$ be the unknown basis of length $t$ sequences under which distribution $\Pr[\cdot]$ has high fidelity and $|B_t^*| = n^* \leq 1/\Delta^*$. Define a distribution $d_t^*$ over futures given by $d_t^*(f) = \EE_{b \in B_t^*}[\Pr[f| b]]$. Define $D^*_t$ to be a diagonal matrix with diagonal entries given by $d_t^*(f)$.

    We ignore the $t$ dependence in notation in this proof for clarity. Before, we prove that $S$ is a robust basis at length $t$, we first show some properties of $d^*$ that will come in handy. First, the norm of $\Pr[F|x]$ under $d$ is upper bounded:
    \begin{equation} \label{eq:prop1}
        \norm{D^{*\frac{-1}{2}} \Pr[F|x]}^2_2 = \EE \sbr{\frac{\Pr[f|x]}{d^*(f)}}^2 = \EE \sbr{\frac{\sum_i \alpha_i(x) \Pr[f|b^*_i] }{d^*(f)}}^2 \leq \EE \sbr{\norm{\alpha(x)}_1 n^* }^2 \leq n^{*3}
    \end{equation} since $\Pr[f|b^*_i]/d^*(f) \leq n^*$ by definition and $\norm{\alpha(x)}_1 \leq \sqrt{n^*} \norm{\alpha(x)}_2 \leq \sqrt{n^*}$ as $\alpha(x)$ are coefficients under basis $B^*$.

    Our next step is to show that the eigenvalues under $S$ and $B^*$ are not so different. For clarity, let \[
        H = D^{*\frac{-1}{2}} \rbr{\EE\sbr{ \Pr[F|x] \Pr[F|x]^\top} - \EE_{s \in S}\sbr{\Pr[F|s] \Pr[F|s]^\top}} D^{*\frac{-1}{2}}
    \]
     Then, by Matrix Bernstein inequality (\citep[Theorem 1.6.2]{tropp2015introduction}) and \Cref{eq:prop1}, for $|S| = n = O(n^{*6} (\Delta^*)^{-2} \log (2r/\delta T))$ and $\Delta^*<1$, we get the following bound on $H$ with probability $1-\delta$: \[
        \norm{H}_2 \leq \sqrt{\frac{2n^{*6}\log r}{2n}} + \frac{2n^{*3}\log 2r}{6n} \leq \frac{\Delta^*}{2}
    \] 
    And therefore by Weyl's inequality, the top $r_t$ eigenvalues of \begin{equation}\label{eq:prop2}
        \frac{1}{n} D^{*\frac{-1}{2}} \rbr{\EE_{s \in S}\sbr{\Pr[F|s] \Pr[F|s]^\top}} D^{*\frac{-1}{2}}
    \end{equation} are $> \Delta^*/2$ and all other eigenvalues are $0$.
    We now show that $S$ forms a basis. For clarity, let $\Pr[F|S]$ be a matrix with $\Pr[F|s_i]$ as columns. Then, $\Pr[F|S]^\top D^{*-1} \Pr[F|S]$ has $r_t$ eigenvalues $> n \Delta^*/2$. Let $V_t$ be the eigenvectors corresponding to non-zero eigenvalues of the latter matrix. Next we note that $\spn(\{\Pr[F|s_1], \ldots, \Pr[F|s_n]\})$ has dimension $r_t$ and therefore there exists coefficients $\beta(x) \in \spn(V_t)$ such that \[
        \Pr[F|x] = \sum_i \beta_i(x) \Pr[F|s_i] = \Pr[F|S] \beta(x)
    \] Multiplying both sides by $D^{*\frac{-1}{2}}$, we get \begin{align*}
        n^{*3/2} \geq \norm{D^{*\frac{-1}{2}} \Pr[F|x]}_2 & = \norm{D^{*\frac{-1}{2}}\Pr[F|S] \beta(x)}_2 \geq \sqrt{\frac{n\Delta^*}{2}}  \norm{\beta(x)}_2
    \end{align*}
    where the first inequality follows from \Cref{eq:prop1} and the last inequality follows from \Cref{eq:prop2}. Simplifying this shows $S$ forms an basis with the following upper bound on the coefficients \begin{equation}
        \norm{\beta(x)}_2 \leq \sqrt{\frac{2n^{*3}}{n\Delta^*}} < 1
    \end{equation}
    The only remaining part is to show that the two distributions $d$ and $d^*$ are only a small factor apart. For this, we first note from \Cref{eq:prop1} \[
        \frac{d(f)}{d^*(f)} = \EE_{s \in S}\sbr{\frac{\Pr[F|s_i]}{d^*(f)}} \leq n^{*3/2}
    \] and therefore non-zero eigenvalues of $D^{\frac{-1}{2}} D^{*\frac{1}{2}}$ are $> 1/n^{*3/4}$. Since, \begin{align*}
         &  D^{\frac{-1}{2}} D^{*\frac{1}{2}}  \rbr{D^{*\frac{-1}{2}}\Pr[F|S] \Pr[F|S]^\top D^{*\frac{-1}{2}}} D^{*\frac{1}{2}} D^{\frac{-1}{2}}  \\
         & = D^{\frac{-1}{2}}\Pr[F|S] \Pr[F|S]^\top D^{\frac{-1}{2}} = \Sigma_S\, ,
    \end{align*} it follows from discussion above and \Cref{eq:prop2} that $\Sigma_S$ has $r_t$ eigenvalues greater than \[
        \frac{n \Delta^*}{2 n^{*3/2}}
    \] Substituting $n^* \leq 1/\Delta^*$ proves the claim.
\end{proof}

\subsection{Learning operators}\label{app:sub:2}
With access to a $\Delta$-robust basis $\{B_t\}_{t\in [T]}$, we now show how to learn an approximation of operators $A_o$. For this, we need to learn approximations of projections $P_{V_t}$ and coefficients $\beta(b_i o)$ for the one-step extensions of basis $B_{t} = \{b_1, \ldots, b_n\}$. Towards this end, using the techniques in \Cref{app:sub-1}, we first estimate the following: \begin{equation}\label{eq:app:qsigma}
    q(x) = \Pr[F_t|B_t]^\top  D_t^{-1} \Pr[F|x]\quad \text{and}\quad \Sigma_{B_t} = \Pr[F_t|B_t]^\top D_t^{-1} \Pr[F_t|B_t]   \, .
\end{equation} Note $\Sigma_{B_t}$ is a covariance matrix with $r_t< r$ non-zero eigenvalues. Let $\widehat \Sigma_{B_t}$ be an approximation for covariance matrix $\Sigma_{B_t}$ from \Cref{cor:estimatesigma}. Compute SVD of $\widehat\Sigma_{B_t}$, and let $\widehat V_{t}$ be the matrix of top $r_t$ eigenvectors (or equivalently corresponding to eigenvalues $> \Delta/2$ according to the proof of proposition below). Then, we claim $\widehat V_{t}$ is close to $V_t$:
\begin{proposition}\label{prop:alg2}
     Let $\Pr[\cdot]$ be any rank $r$ distribution over observation sequences of length $T$. Assume knowledge of a $\Delta$-robust basis $\{B_t\}_{t\in [T]}$ for distribution $\Pr[\cdot]$. Let $\Sigma_{B_t}, \widehat \Sigma_{B_t}, V_t$ and $\widehat V_t$ be as defined above. Then, the approximate projection $P_{\widehat V_t}$ satisfies \[
        \norm{P_{V_t} - P_{\widehat V_t}}_F \leq O\rbr{\frac{\sqrt{r} \cdot \norm{\widehat \Sigma_{B_t} - \Sigma_{B_t}}_2}{\Delta}}
    \] for $\norm{\widehat \Sigma_{B_t} - \Sigma_{B_t}}_2 \leq \Delta/ 2$.
\end{proposition}
\begin{proof}
    Let estimate $\Sigma_{B_t}$ be such that \begin{equation}
        \label{eq:prop4-1}
        \norm{\widehat \Sigma_{B_t} - \Sigma_{B_t}}_2 \leq \norm{\widehat \Sigma_{B_t} - \Sigma_{B_t}}_F \leq \alpha 
    \end{equation}
    The claim now follows from Davis-Kahan $\sin(\theta)$ theorem. From our assumptions, eigenvalues of $V_t$ are $> \Delta$ and rest are all $0$. Moreover, from \Cref{eq:prop4-1} and Weyl's inequality, all the eigenvalues associated to $\widehat V^\perp_t$ are $< \alpha$. Then, for $\alpha < \Delta/2$, we get using \Cref{cor:project},  \[
        \norm{P_{V_t} - P_{\widehat V_t}}_F \leq \frac{2\sqrt{2r} \alpha}{\Delta}
    \]
    
\end{proof}

To compute an approximation for operators $A_{o,t}$, we also need to get the coefficients for one-step extensions of elements in the basis.
\begin{proposition}\label{prop:alg3}
    Let $\Pr[\cdot]$ be any rank $r$ distribution over observation sequences of length $T$. Assume knowledge of a $\Delta$-robust basis $\{B_t\}_{t\in [T]}$ for distribution $\Pr[\cdot]$. Also, assume the coefficients under basis $\{B_t\}_{t\in [T]}$ are bounded i.e. $\norm{\beta(h)}_2 \leq c$ for all histories $h \in \obs^{\leq T}$ and some constant $c \leq 1$. Suppose we have estimates $\widehat q(b_i o)$ and $\widehat \Sigma_{B_t}$ such that  \[
        \max\cbr{\norm{\widehat q(b_i o) - q(b_i o)}_2 \, ,\norm{\widehat \Sigma_{B_t} - \Sigma_{B_t}}_2} \leq \alpha\, .  
    \]
     Define $\widehat \beta(b_i o)$, approximation of $\beta(b_i o)$, for $\lambda = 4 \alpha^2/c^2$: \[
        \widehat \beta(b_i o) = \argmin_z \norm{\widehat \Sigma_{B_t} - \widehat q(b_i o)}^2_2 + \lambda \norm{z}_2^2\, .
    \] Then, the following are true: \begin{enumerate}
        \item the norm of the approximation $\widehat \beta(b_i o)$ is bounded:\[
            \norm{\widehat \beta(b_i o)}_2\leq \sqrt{2} c
        \]
        \item the approximation $\widehat \beta(b_i o)$ is accurate in the $\spn(V_{t+1})$:\[
            \norm{P_{V_{t+1}} \beta(b_i o) - P_{V_{t+1}} \widehat \beta(b_i o)}_2 \leq O\rbr{\frac{\alpha}{\Delta}}\, .
        \]
    \end{enumerate} 
\end{proposition}
\begin{proof}
    Let \begin{align*}
        \norm{\widehat q(b_i o) - q(b_i o)}_2 \, ,\norm{\widehat \Sigma_{B_t} - \Sigma_{B_t}}_2 &\leq \alpha\\
        \lambda &= 4\alpha^2/c^2
    \end{align*}
    We know that $\beta(b_i o)$ satisfies \[
        \Sigma_{B_t} \beta(b_i o) = q(b_i o)
    \] and approximate it by the following program (for $\lambda = 4 \alpha^2$): \[
        \widehat \beta(b_i o) = \argmin_z \norm{\widehat \Sigma_{B_t} z - \widehat q(b_i o)}^2_2 + \lambda \norm{z}^2
    \] First, we can see that the above error is small as \[
        \norm{\widehat \Sigma_{B_t} \widehat \beta(b_i o) - \widehat q(b_i o)}^2_2 + \lambda \norm{\widehat \beta(b_i o)}_2^2 \leq \norm{\widehat \Sigma_{B_t} \beta(b_i o) - \widehat q(b_i o)}^2_2 + \lambda \norm{\beta(b_i o)}_2^2\leq 8\alpha^2
    \] as $\norm{\beta(b_i o)}_2 \leq c$ and \[
        \norm{\widehat \Sigma_{B_t} \beta(b_i o) - \widehat q(b_i o)}_2 \leq \norm{\widehat \Sigma_{B_t} \beta(b_i o) - \Sigma_{B_t} \beta(b_i o)}_2 + \norm{q(b_i o) - \widehat q(b_i o)}_2 \leq 2\alpha
    \] Note that this also means \[
        \norm{\widehat \beta(b_i o)} \leq \sqrt{\frac{8 \alpha^2 c^2}{4\alpha^2}} = \sqrt{2} c
    \]
    Using these, we get \begin{align*}
         & \norm{\Sigma_{B_t} \beta(b_i o) - \Sigma_{B_t} \widehat \beta(b_i o)}                                                                                                          \\
         & \leq \norm{\Sigma_{B_t} \beta(b_i o) - \widehat \Sigma_{B_t} \widehat \beta(b_i o) + \widehat \Sigma_{B_t} \widehat \beta(b_i o) - \Sigma_{B_t} \widehat \beta(b_i o)}                                     \\
         & \leq \norm{\Sigma_{B_t} \beta(b_i o) - \widehat \Sigma_{B_t} \widehat \beta(b_i o)} + \norm{\widehat \Sigma_{B_t} \widehat \beta(b_i o) - \Sigma_{B_t} \widehat \beta(b_i o)}                              \\
         & \leq \norm{\Sigma_{B_t} \beta(b_i o) - q(b_i o) + q(b_i o) - \widehat q(b_i o) + \widehat q(b_i o) - \widehat \Sigma_{B_t} \widehat \beta(b_i o)} + \norm{\widehat \Sigma_{B_t} - \Sigma_{B_t}} \norm{\widehat \beta(b_i o)}               \\
         & \leq \norm{\Sigma_{B_t} \beta(b_i o) - q(b_i o)} + \norm{q(b_i o) - \widehat q(b_i o)} + \norm{\widehat q(b_i o) - \widehat \Sigma_{B_t} \widehat \beta(b_i o)} + \norm{\widehat \Sigma_{B_t} - \Sigma_{B_t}} \norm{\widehat \beta(b_i o)} \\
         & \leq \alpha + 2\sqrt{2} \alpha + \sqrt{2} c \alpha \leq 6 \alpha
    \end{align*}
    for $c<1$. Now, using that our assumption on $\Sigma_{B_t}$, \[
        \norm{P_{V_{t+1}} \beta(b_i o) - P_{V_{t+1}} \widehat \beta(b_i o)}_2 \leq \frac{6 \alpha}{\Delta} 
    \]
\end{proof}

\begin{proposition}\label{prop:alg4}
    Let $\Pr[\cdot]$ be any rank $r$ distribution over observation sequences of length $T$ and $\{B_t\}_{t\in [T]}$ be some basis of distribution $\Pr[\cdot]$. Then, we can build an estimate $\diag(\widehat \Pr[o|B_t])$ for every $o \in \obs$ and $t \in [T]$, such that with probability $1-\delta$  \[
        \|\diag(\Pr[o|B_t]) - \diag(\widehat \Pr[o|B_t])\|_2 \leq \varepsilon
    \] using $\widetilde O(|\obs| T n^3/\varepsilon^2 \log(1/\delta))$ conditional samples.
\end{proposition} \begin{proof}
    This follows from Hoeffding's inequality (\Cref{prop:hoeffding}).
\end{proof}

Using the approximations above, we are in position to present the approximation error for our estimate of operator $A_{o, t}$.
\begin{lemma}[Restatement of \Cref{lemma:cond:1}]\label{lemma:error-decom}
    Let $\Pr[\cdot]$ be any rank $r$ distribution over observation sequences of length $T$. Assume knowledge of a $\Delta$-robust basis $\{B_t\}_{t\in [T]}$ for distribution $\Pr[\cdot]$. Also, assume the coefficients under basis $\{B_t\}_{t\in [T]}$ are bounded i.e. $\norm{\beta(h)}_2 \leq c < 1$ for all histories $h \in \obs^{\leq T}$ and some constant $c \leq 1$. Then, we can build an estimate $\widehat A_{o, t}$ using \[ \widetilde O\rbr{\frac{c^8 r^{3} n^{21} |\obs|^3 T^5}{\Delta^6 \varepsilon^6} \log^2\rbr{\frac{1}{\delta}}}
     \]  many conditional samples such that with probability $1-\delta$, for any vector $v$ \begin{align} \label{eq:lem-error-1a}
        (\widehat A_{o, t} - A_{o, t}) v = \beta(B_{t+1}) \alpha(o,v) + V_{t+1}^\perp \alpha^\perp(o,v)
    \end{align} where $\beta(B_{t+1})$ is a matrix with its columns given by coefficients $\beta(b_i)$ for $B_{t+1} = \{b_1, \ldots, b_n\}$ and $V_{t+1}^\perp$ is a matrix with its columns given by the right singular vectors corresponding to kernel of $\Pr[F_{t+1} | B_{t+1}]^\top D_{t+1}^{-1} \Pr[F_{t+1} | B_{t+1}]$. Moreover, the column vectors $\alpha(o,v)$ and $\alpha^\perp(o,v)$ are $\ell_1$ bounded i.e. \begin{equation}\label{eq:lem-error-1b}
        \|\alpha(o,v)\|_1, \|\alpha^\perp(o,v)\|_1 \leq \varepsilon \norm{v}_2\, .
    \end{equation}
\end{lemma}
\begin{proof} Using $P_{\widehat V_{t}}$, $\widehat \beta(B_t o)$ and $\diag(\widehat \Pr[o|B_t])$ from \Cref{prop:alg2,prop:alg3,prop:alg4}, define operator $\widehat A_{o, t}$, an approximation of $A_{o, t}$, as \begin{equation*}
        \widehat A_{o, t} = P_{\widehat V_{t+1}} \widehat \beta(oB_t)  \diag( \widehat \Pr[o|B_t]) P_{\widehat V_t}
    \end{equation*} 
    Let \begin{align*}
        \norm{P_{\widehat V_{t+1}} - P_{V_{t+1}}}_2 \, ,\norm{P_{\widehat V_{t}} - P_{V_{t}}}_2                    & \leq \alpha_1 = O\rbr{\frac{\sqrt{r} n \varepsilon}{\Delta}}     \\
        \norm{P_{V_{t+1}}\widehat \beta(B_t o) - P_{V_{t+1}}\beta(B_t o)}_2  & \leq \alpha_2 \leq \sqrt{n} \max_{b \in B_t} \norm{P_{V_{t+1}}\widehat \beta(b o) - P_{V_{t+1}}\beta(b o)}_2 \leq  O\rbr{\frac{n^{3/2} \varepsilon}{\Delta}} \\
       \norm{\diag(\widehat \Pr[o|B_t]) - \diag(\Pr[o|B_t])}_2 & \leq \alpha_3 = \varepsilon
    \end{align*} Also, note that $\norm{\widehat \beta(b_i o)}_2 \leq \sqrt{2} c$, $\norm{\diag( \widehat \Pr[o|B_t])}_2 \leq 2$ and $\norm{ P_{\widehat V_{t+1}}}_2\, , \norm{ P_{\widehat V_{t}}}_2 \leq 1$. To prove our main claim, we will first show that for any unit vector $v$ \begin{equation*}
        (\widehat A_{o, t} - A_{o, t}) v = V_{t+1} e(o,v) + V_{t+1}^\perp \alpha^\perp(o,v)
    \end{equation*} with $\|e(o,v)\|_2 , \|\alpha^\perp\|_2 \leq 4\sqrt{2} c \alpha_1 + \sqrt{2} c \alpha_3 + \alpha_2$. 
    To prove this, we will show that \begin{align*}
         & \nbr{\widehat A_{o, t} - A_{o, t}}                                                                                                                                                              \\
         & =   \nbr{P_{\widehat V_{t+1}} \widehat \beta(oB_t) \diag(\widehat  \Pr[o|B_t]) P_{\widehat V_{t}} - P_{V_{t+1}} \beta(oB_t) \diag( \Pr[o|B_t]) P_{V_{t}}}                                                   \\
         & = \nbr{(P_{\widehat V_{t+1}} - P_{V_{t+1}}) \widehat \beta(oB_t)  \diag( \widehat \Pr[o|B_t]) P_{\widehat V_{t}}} + \nbr{P_{V_{t+1}} \widehat \beta(oB_t)  \diag( \widehat \Pr[o|B_t]) (P_{\widehat V_{t}} -P_{V_{t}})} \\
         & + \nbr{P_{V_{t+1}} \widehat \beta(oB_t) ( \diag( \widehat \Pr[o|B_t]) - \diag( \Pr[o|B_t]))P_{V_{t}}} + \nbr{P_{V_{t+1}}(\widehat \beta(oB_t) - \beta(oB_t)) \diag( \Pr[o|B_t])P_{V_{t}}}
    \end{align*} The first term is bounded by $ 2\sqrt{2} c \alpha_1$, second term by $ 2\sqrt{2} c\alpha_1$, similarly third term  by $\sqrt{2} c \alpha_3$ and the last term by $\alpha_2$.
    Therefore, we get that \begin{align}
        \nbr{\widehat A_{o, t} - A_{o, t}}_2 \leq 4\sqrt{2} c \alpha_1 + \sqrt{2} c \alpha_3 + \alpha_2
    \end{align} which implies for any unit vector $v$ \begin{equation}
        (\widehat A_{o, t} - A_{o, t}) v = V_{t+1} e(o,v) + V_{t+1}^\perp \alpha^\perp(o,v)
    \end{equation} with $\|e(o,v)\|_2, \|\alpha^\perp(o,v)\|_2 \leq \|\widehat A_{o, t} - A_{o, t}\|_2 \leq 4\sqrt{2} c \alpha_1 + \sqrt{2} c \alpha_3 + \alpha_2$. To complete our proof, we will show that  \begin{equation}
        V_{t+1} = \beta(B_{t+1}) V_{t+1}
    \end{equation} This is enough, as this gives, \[
        (\widehat A_{o, t} - A_{o, t}) v = \beta(B_{t+1}) \alpha(o,v) + V_{t+1}^\perp \alpha^\perp(o,v)
    \] where $\alpha(o, v) = V_{t+1} e(o,v)$ as $\norm{\alpha(o, v)}_2 \leq \norm{V_{t+1} e(o,v)} \leq \norm{e(o,v)}$ as $V_{t+1}$ is a unit norm matrix.

    We now prove our claim. First note that $\Pr[F|B_{t+1}] I = \Pr[F|B_{t+1}]$ and therefore by uniqueness of $\beta(B_{t+1})$, \begin{align*}
        \beta(B_{t+1})                  & = V_{t+1} V_{t+1}^\top I                                                         \\
        \implies \beta(B_{t+1}) V_{t+1} & = V_{t+1} \tag{by right multiplying by $V_{t+1}$ and $V_{t+1}^\top V_{t+1} = I$}
    \end{align*}
    This whole process requires $
    O(n^2 |\obs| T m_\varepsilon)$ many conditional samples with $m_\varepsilon$ defined in \Cref{prop:estimate-sum1}. Substituting in $\varepsilon = \Delta \varepsilon'/ (c \sqrt{r} n^{3/2})$ gives the result.
\end{proof}

\subsection{Perturbation analysis: Error in coefficients} \label{app:sec:perturb}
Our approach for learning the distribution $\Pr[\cdot]$ is to learn approximations of operators $\widehat A_{o,t}$ and use them to compute probabilities\begin{equation*}
    \widehat \Pr[x_{1:T}] = \widehat A_{x_T, T-1} \widehat A_{x_{T-1}, T-2} \ldots \widehat A_{x_1, 0}
  \end{equation*} For clarity, let $A_{x_{1:t}}$ and $\widehat A_{x_{1:t}}$ represent the product of matrices $A_{x_t, t-1} \ldots A_{x_1,0}$ and $\widehat A_{x_t, t-1} \ldots \widehat A_{x_1,0}$ respectively. Similarly, we use $x_{1:t}$ to represent the sequence $(x_1,\ldots, x_t)$. In this section, we now present a technical lemma showing that the errors in our estimates of  $\widehat \Pr[x_{1:T}]$ scales linearly with errors in our approximate operators $\widehat A_{o,t}$.
\begin{proposition} \label{prop:perturb} Suppose we have estimates $\widehat A_{o,t}$ for every $t \in [T]$ and observation $o \in \obs$ which satisfy the conditions \Cref{eq:lem-error-1a,eq:lem-error-1b} in \Cref{lemma:error-decom}.
    Then, for every sequence $(x_1, \ldots, x_t)$ with $t \in [T]$, the product $\widehat A_{x_{1:t}}$ satisfies: 
    \[(\widehat A_{x_{1:t}} - A_{x_{1:t}}) \beta(\varphi) = \beta(B_{x_{1:t}}) \gamma_{x_{1:t}} + V^\perp \gamma^\perp_{x_{1:t}}
    \] where $B_{x_{1:t}} \subset H_{t+1}$ i.e. a subset of histories of length $t+1$ (and of size possibly exponential in $t$). Moreover, the $\ell_1$ norm of the associated coefficients $\gamma_{x_{1:t}}$ and $\gamma^\perp_{x_{1:t}}$ grow moderately:\begin{align}\label{eq:to-prove}
        \sum_{x_t \in \obs} \|\gamma_{x_{1:t}}\|_1       & \leq |\obs|\varepsilon \Pr[x_{1: t-1}] + (1+|\obs|\varepsilon) \|\gamma_{x_{1:t-1}}\|_1+  |\obs|\varepsilon \|\gamma^\perp_{x_{1:t-1}}\|_1                    \\
        \sum_{x_t \in \obs} \|\gamma^\perp_{x_{1:t}}\|_1 & \leq |\obs|\varepsilon \Pr[x_{1: t-1}] + |\obs|\varepsilon \|\gamma_{x_{1:t-1}}\|_1 +  |\obs|\varepsilon \|\gamma^\perp_{x_{1:t-1}}\|_1 \label{eq:to-prove-2}
    \end{align}
\end{proposition}
\begin{proof}
First, by our assumption, for any observation sequence $x = (x_1,\ldots, x_{t-1})$  \begin{equation}\label{eq:needful}
        (\widehat A_{o, t-1} - A_{o, t-1}) \beta(x) = \beta(B_{t}) \alpha(o,x) + V_{t}^\perp \alpha^\perp(o,x)
    \end{equation}
with $\|\alpha(o,x)\|_1, \|\alpha^\perp(o,x)\|_1 \leq \varepsilon$ as $\norm{\beta(x)}_2 \leq 1$.
    Next, the following recursive relation holds\begin{align*}
         & (\widehat A_{x_{1:t}} - A_{x_{1:t}}) \beta(\varphi)                                                                                                                                                                    \\
         & =  (\widehat A_{x_{1:t}} - \widehat A_{x_t} A_{x_{1:t-1}} + \widehat A_{x_t} A_{x_{1:t-1}} - A_{x_{1:t}}) \beta(\varphi)                                                                                                       \\
         & = (\widehat A_{x_{t}} - A_{x_{t}}) A_{x_{1:t-1}} \beta(\varphi) +  \widehat A_{x_t} (\widehat A_{x_{1:t-1}} - A_{x_{1:t-1}}) \beta(\varphi)                                                                                \\
         & = (\widehat A_{x_{t}} - A_{x_{t}}) A_{x_{1:t-1}} \beta(\varphi) +  A_{x_t} (\widehat A_{x_{1:t-1}} - A_{x_{1:t-1}}) \beta(\varphi) + (\widehat A_{x_t} - A_{x_t})  (\widehat A_{x_{1:t-1}} - A_{x_{1:t-1}}) \beta(\varphi)
    \end{align*}
    We will bound the three terms separately for each $x_t$. We start by bounding the first term\begin{align}
         & (\widehat A_{x_{t}} - A_{x_{t}}) A_{x_{1:t-1}} \beta(\varphi)    \notag                                                                            \\
         & = (\widehat A_{x_{t}} - A_{x_{t}}) \Pr[x_{1: t-1}] \beta(x_{1:t-1})\tag{by \Cref{prop:main1}}                                                        \\
         & = \Pr[x_{1: t-1}] \beta(B_{t}) \alpha(x_t,x_{1:t-1}) +  \Pr[x_{1: t-1}]V_{t}^\perp\alpha^\perp(x_t,x_{1:t-1})\tag{by \Cref{eq:needful}} \\
         & = \beta(B_{t}) \left(\Pr[x_{1: t-1}]  \alpha(x_t, x_{1:t-1})\right) + V_{t}^\perp\left( \Pr[x_{1: t-1}]\alpha^\perp(x_t,x_{1:t-1})\right)\notag
    \end{align} and here we can see that \begin{align*}\|\Pr[x_{1: t-1}]  \alpha(x_t,x_{1:t-1})\|_1 \leq \Pr[x_{1: t-1}]  \|\alpha(x_t,x_{1:t-1})\|_1 \leq \varepsilon \Pr[x_{1: t-1}] \\
        \|\Pr[x_{1: t-1}]  \alpha^\perp(x_t,x_{1:t-1})\|_1 \leq \Pr[x_{1: t-1}]  \|\alpha^\perp(x_t, x_{1:t-1})\|_1 \leq \varepsilon \Pr[x_{1: t-1}]
    \end{align*}
    where we used $\norm{\alpha^\perp(x_t,x_{1:t-1})}_1, \norm{\alpha(x_t,x_{1:t-1})}_1 \leq \varepsilon \norm{\beta(x_{1:t-1})}_2 \leq \varepsilon$. This gives the first term in \Cref{eq:to-prove,eq:to-prove-2} (with the $|\obs|$ factor for sum over $x_t$).
    We bound the remaining terms by induction. Assume \begin{equation}\label{eq:hypo}
        (\widehat A_{x_{1:t-1}} - A_{x_{1:t-1}}) \beta(\varphi) = \beta(B_{x_{1:t-1}}) \gamma_{x_{1:t-1}} + V_{t-1}^\perp \gamma^\perp_{x_{1:t-1}}
    \end{equation} where $B_{x_{1:t-1}}$ is a set of observation sequences of length $t-1$. Let's first bound the second term. \begin{align}
         & A_{x_t} (\widehat A_{x_{1:t-1}} - A_{x_{1:t-1}}) \beta(\varphi)                                            \notag    \\
         & =A_{x_t} (\beta(B_{x_{1:t-1}}) \gamma_{x_{1:t-1}} + V_{t-1}^\perp \gamma^\perp_{x_{1:t-1}})\tag{by \Cref{eq:hypo}} \\
         & =A_{x_t} (\beta(B_{x_{1:t-1}}) \gamma_{x_{1:t-1}})\tag{by \Cref{prop:main1}}                                    \\
         & =\beta(x_tB_{x_{1:t-1}}) \diag(\Pr[x_t|B_{x_{1:t-1}}]) \gamma_{x_{1:t-1}}\tag{by \Cref{prop:main1}}
    \end{align}
    We can bound the $\ell_1$ norm of the coefficients as\begin{align*}
        \sum_{x_t\in \obs}\sum_{b_{i} \in B_{{x_{1:t-1}}}} | \Pr[x_t | b_{i}] \gamma_{i}| & = \sum_{x_t\in \obs}\sum_{b_{i} \in B_{x_{1:t-1}}} \Pr[x_t | b_{i}] |\gamma_{i}| = \sum_{b_{i} \in B_{x_{1:t-1}}}  |\gamma_{i}| = \|\gamma_{x_{1:t-1}}\|_1
    \end{align*} where the second last step follows from $\sum_{x_t\in \obs} \Pr[x_t | b_{i}] = 1$. For clarity, let $\alpha(x_t, B_{x_{1:t-1}})$ represent a matrix with its column given by $\alpha(x_t, b)$ for $b \in B_{x_{1:t-1}}$. Similarly, define $\alpha^\perp(x_t,  B_{x_{1:t-1}})$, $\alpha(x_t, V_{t-1}^\perp)$ and $\alpha^\perp(x_t,  V_{t-1}^\perp)$. We can then similarly bound the remaining term.
    \begin{align}
         & (\widehat A_{x_t} - A_{x_t})  (\widehat A_{x_{1:t-1}} - A_{x_{1:t-1}}) \beta(\varphi)                                                                                            \notag      \\
         & =(\widehat A_{x_t} - A_{x_t}) (\beta(B_{x_{1:t-1}}) \gamma_{x_{1:t-1}} + V_{t-1}^\perp \gamma^\perp_{x_{1:t-1}})\tag{by \Cref{eq:hypo}}                                                    \\
         & = \Big[\beta(B_{t}) \alpha(x_t, B_{x_{1:t-1}}) \gamma_{x_{1:t-1}} + V_{t}^\perp \alpha^\perp(x_t,  B_{x_{1:t-1}}) \gamma_{x_{1:t-1}} \Big]                                  \tag{by \Cref{eq:needful}}   \\
         & + \Big[\beta(B_{t}) \alpha(x_t, V_t^\perp) \gamma^\perp_{x_{1:t-1}} + V_{t}^\perp \alpha^\perp(x_t,  V_t^\perp) \gamma^\perp_{x_{1:t-1}} \Big] \tag{by \Cref{lemma:error-decom}}
    \end{align} Each of these terms can be bounded similarly. We show how to bound the first term: \begin{align}
         & \|\sum_{b_i \in B_{x_{1:t-1}}} \gamma_{i}  \alpha(x_t, b_i)\|_1\leq \sum_{b_i \in B_{x_{1:t-1}}} |\gamma_{i}| \cdot  \|\alpha(x_t, b_i)\|_1 \leq \varepsilon \|\gamma_{t-1}\|
    \end{align} where we used $\norm{\alpha(x_t, b_i)}_1 \leq \varepsilon \norm{\beta(b_i)}_2 \leq \varepsilon$. This gives the remaining terms in \Cref{eq:to-prove,eq:to-prove-2} (with the $|\obs|$ factor for sum over $x_t$).
\end{proof}

We next give a solution for the recursion from \Cref{prop:perturb}. This is standard, but we give a proof for completeness.
\begin{proposition}
    \label{prop:recurse}
    Consider the following recursions: \begin{align*}
        f(0) = 0; g(0) =0                                \\
        f(t) = d\varepsilon + (1 +  d \varepsilon) f(t-1) + d \varepsilon g(t-1) \\
        g(t) = d \varepsilon + d \varepsilon f(t-1) + d \varepsilon g(t-1)
    \end{align*} Let $T\in \ZZ^+$. Then, the following holds for all $t \leq T$ and $\varepsilon \leq 1/12 d T$:\begin{align*}
        f(t) \leq 3dT\varepsilon
    \end{align*}
\end{proposition}
\begin{proof}
    We first claim: $g(t) \leq 6d\varepsilon$ for all $t \leq T$. We prove this by strong induction. This is true for $t=1$. Let's assume $g(i) \leq 6\varepsilon$ for all $i \leq t-1$. Then, first we unroll the recursion for $f(t)$. \begin{align*}
        f(t) & = d \varepsilon + (1 +  d \varepsilon) f(t-1) + d \varepsilon g(t-1)                                                                       \\
             & = d \varepsilon + (1 +  d \varepsilon) \sbr{d \varepsilon + (1 +  d \varepsilon) f(t-2) + d \varepsilon g(t-2)} + d \varepsilon g(t-1)              \\
             & = d \varepsilon + (1 +  d \varepsilon)d \varepsilon + (1 +  d \varepsilon)^2 f(t-2) + d \varepsilon (1 +  d \varepsilon) g(t-2) + d \varepsilon g(t-1) \\
             & = (1 +  d \varepsilon)^t - 1 + d\varepsilon \rbr{\sum_{i=1}^{t-1} (1 +  d \varepsilon)^{i-1} g(t-i)}
    \end{align*} where the last equation follows from observing the first few terms form a geometric series. Using our induction hypothesis, we get \begin{align*}
        f(t-1) \leq f(t) & \leq (1 +  2 d T \varepsilon) - 1 + (d\varepsilon) (6 d \varepsilon)\rbr{\sum_{i=1}^{t-1} (1 +  d \varepsilon)^{i-1}} \\
                         & \leq 2d T\varepsilon + 6d\varepsilon \rbr{(1 +  d\varepsilon)^{t-1} - 1}                \\
                         & \leq 2d T\varepsilon  + 6d\varepsilon (2dT\varepsilon)                                             \\
                         & \leq 3dT\varepsilon
    \end{align*} where we used that $(1 + a)^t \leq 1 + 2at$ for $a< 1/2t$ and $\varepsilon < 1/12dT$. And therefore, we can bound $g(t)$ as \begin{align*}
        g(t) & = d \varepsilon + d \varepsilon f(t-1) + d \varepsilon g(t-1) \\
             & \leq  d \varepsilon + 3d^2T\varepsilon^2  + 6d^2\varepsilon^2     \\
             & \leq 3d \varepsilon
    \end{align*}
    where we used $g(t-1) \leq 6d\varepsilon$ and $\varepsilon \leq 1/12dT$. This proves that $g(t) \leq 6 d \varepsilon$ for all $t< T$. Moreover, by arguments above, in that case $f(t) \leq 3dT\varepsilon$ for all $t \leq T$.
\end{proof}

\begin{lemma}[Restatement of \Cref{lemma:error-prop-1}]\label{prop:last}
    Let $\widehat A_{o, t}$ be the approximation of $A_{o,t}$ from \Cref{lemma:error-decom}. Furthermore, for sequence $x_{1:T}$ of length $T$, define \[
        \widehat \Pr[x_{1:T}] = \widehat A_{x_{1:T}}\, .
        \] Then, for the values of $\varepsilon < (12|\obs| T)^{-1} $, we get that the functions $\Pr$ and $\widehat \Pr$ are close in TV distance:
    \begin{align*}
        TV(\Pr, \widehat \Pr) \leq 2 |\obs| T \varepsilon\, .
    \end{align*}
\end{lemma}
\begin{proof} Recall $F_T = \{\varphi\}$. Then,
    \begin{align}
        2 \cdot TV(\Pr, \widehat \Pr) & = \sum_{x_{1:T}} \abr{\widehat \Pr[x_{1:T}] - \Pr[x_{1:T}]}              \notag                                                      \\
                            & = \sum_{x_{1:T}} \abr{\widehat A_{x_{1:T}} - A_{x_{1:T}}}    \notag              \\
                            & = \sum_{x_{1:T}} \abr{\beta(B_{x_{1:T}}) \gamma_{x_{1:T}} + V_{T}^\perp \gamma^\perp_{x_{1:T}}}         \tag{by \Cref{prop:perturb}} \\
                            & = \sum_{x_{1:T}} \abr{\beta(B_{x_{1:T}}) \gamma_{x_{1:T}}} \tag{as $V_{T}^\perp = \ker(\Pr[F_T|B_{T}]) = [0]$(\Cref{prop:vperp})}         \\
                            & =\sum_{x_{1:T}} \abr{\ones \gamma_{x_{1:T}}} \tag{as $\beta(x) = 1$ for all $x \in H_T$}                                 \\
                            & \leq  \sum_{x_{1:T}} \|\gamma_{x_{1:T}}\|_1
    \end{align} The claim follows from \Cref{prop:perturb} and \Cref{prop:recurse}.
\end{proof}

\subsection{Proof of \Cref{thm:main}}\label{sec:main-result}
\thmtwo*
\begin{proof}
    From \Cref{prop:1}, wp $1/2$, we built a $\Delta$-robust basis of size $n$ where \begin{align*}
        \Delta &= \Omega\rbr{\frac{\log (r/\delta T)}{\Delta^{* 11/2}}}\\
        n &= O\rbr{\frac{\log (r/\delta T)}{\Delta^{*8}}}\\
        c &= O\rbr{\sqrt{\frac{\Delta^{*4} }{\log (r/\delta T) } }}
    \end{align*} using $n T$ samples from $\Pr[\cdot]$. From \Cref{prop:last}, to get TV error $\varepsilon$ w.p. $1/2$, when given access to a $\Delta$-robust basis of size $n$, we need to learn the estimate operators $\widehat A_{o,t}$ using \Cref{lemma:error-decom} to accuracy \[
        \varepsilon' = \frac{\varepsilon}{2 |\obs| T}
    \] Substituting this in \Cref{lemma:error-decom} gives the required sample complexity as \[
        \widetilde O\rbr{\frac{c^8 r^3 n^{21} |\obs|^3 T^5}{\Delta^6 {\varepsilon'}^6} \log^2\rbr{\frac{1}{\delta}}} = \widetilde O\rbr{\frac{c^{8} r^3 n^{21} |\obs|^9 T^{11}}{\Delta^{6} {\varepsilon}^6} \log^2\rbr{\frac{1}{\delta}}} 
    \] Substituting the values of $\Delta, n$ and $c$ above, we need \[
        m = \widetilde O\rbr{\frac{r^3 |\obs|^9 T^{11}}{\varepsilon^6 \Delta^{*119}} \log^2\rbr{\frac{1}{\delta}}} 
      \] many queries to conditional sampling oracle.
\end{proof}

\subsection{Properties of operator $A_o$}
Let $\{B_t\}_{t\in [T]}$ be some basis of distribution $\Pr[\cdot]$ (as defined in \Cref{def:bases}). In \Cref{prop:toy:1}, we defined the operators $A_{o,t}$ under basis $\{B_t\}_{t\in [T]}$ and now prove some properties of this operator and the associated coefficients. 

Consider the eigenvalue decomposition of the covariance matrix associated to $B_t$ where $D_t$ is a diagonal matrix of size $|F_t| \times |F_t|$ with entries $d_t(f) := \EE_{b \in B_t} \Pr[f | b]$ on the diagonal:
\begin{equation*}
    \Pr[F_t|B_t]^\top D_t^{-1} \Pr[F_t|B_t] = \begin{bmatrix}
        V_t & V^{\perp}_t
    \end{bmatrix} \begin{bmatrix}
        M_t & 0 \\
        0   & 0
    \end{bmatrix} \begin{bmatrix}
        V_t^\top \\V^{\perp \top}_t
    \end{bmatrix}
\end{equation*}
Here, $M_t$ is a set of all non-zero eigenvalues, $V_t$ is the eigenspace corresponding to non-zero eigenvalues and $V^\perp_t$ is the eigenspace corresponding to zero eigenvalues. Note that $d_t(\cdot)$ is a distribution. 

\begin{remark}\label{rem:zero}
    We note a mild notational issue here. Since we invert the diagonal matrix $D_{t+1}$, we need to be careful about futures $f$ where $d_{t+1}(f) = 0$. This is not an issue however as by definition of a basis, $d_{t+1}(f) = 0$ implies $\Pr[f|h] = 0$ for all histories $h \in H_{t+1}$. For now, we use $\Pr[F_t|B_t]^\top D_t^{-1} \Pr[F_t|B_t]$ to denote the sum \[
        \sum_{f \in \dom(t)}\sbr{\frac{\Pr(f| B_t)^\top \Pr(f|B_t)}{d_t(f)}} 
    \] where $\dom(t)$ is the set of all futures where $d_t(f) \neq 0$.
  \end{remark}

A basic property of $\spn(V^{\perp}_t)$ follows.
\begin{proposition}\label{prop:vperp}
    $\spn(V^{\perp}_t) = \ker( \Pr[F_t|B_t]^\top D_t^{-1} \Pr[F_t|B_t]) = \ker(\Pr[F_t|B_t])$
\end{proposition}
\begin{proof}
    Define $d_t^*$ be the restriction of distribution $d_t$ over the set $\dom(t)$. Consider any vector $v \in \ker(\Pr[F_t|B_t]^\top D_t^{-1} \Pr[F_t|B_t])$. Then, \begin{align}
      \|\Pr(F_t|B_t) v\|_1 &= \sum_{f} \abr{\Pr(f|B_t) v} \notag\\
      &= \EE_{f \sim d_t^*(\cdot)}\sbr{\abr{\frac{\Pr(f| B_t) v}{d_t(f)}}} \tag{as $d_t^*(f) = d_t(f)$ for $f \in \dom(t)$}\\
      &\leq \sqrt{\EE_{f \sim d_t^*(\cdot)}\sbr{\abr{\frac{\Pr(f| B_t) v}{d_t(f)}}^2}}\\
      &= \sqrt{\EE_{f \sim d_t^*(\cdot)}\sbr{\frac{v^\top\Pr(f| B_t)^\top \Pr(f|B_t)v}{d_t^2(f)}}}\\
      &= \sqrt{v^\top\sum_{f \in \dom(t)}\sbr{\frac{\Pr(f| B_t)^\top \Pr(f|B_t)}{d_t(f)}}v}\\
      &= \sqrt{v^\top \Pr[F_t|B_t]^\top D_t^{-1} \Pr[F_t|B_t] v} = 0
    \end{align} This shows that $\ker(\Pr[F_t|B_t]^\top D_t^{-1} \Pr[F_t|B_t]) \subset \ker(\Pr(F_t|B_t))$. The other direction follows because the definition of $\Pr[F_t|B_t]^\top D_t^{-1} \Pr[F_t|B_t]$ as sum above. This completes the proof.
  \end{proof}
Recall we denote the coefficients associated to history $x$ of length $t$ under basis $B_t$ by $\beta(x)$ given by: \begin{equation}
    \label{eq:beta1}
    \Pr[F_t|B_t] \beta(x) = \Pr[F_t|x]\, .
    \end{equation}
By \Cref{prop:vperp}, we can assume that $\beta(x) \in \spn(V_t)$ without loss of generality. We now show that the coefficients $\beta(x)$ satisfy some nice properties.
\begin{proposition}\label{prop:beta}
    Let $\beta(x) \in \spn(V_t)$ be coefficients associated to history $x$. Then, the following statements are true:\begin{enumerate}
        \item \label{item:beta2} The coefficients  $\beta(x)$ are uniquely defined in $\spn(V_t)$. Formally, Let $\beta'(x)$ be any other coefficients which satisfy \Cref{eq:beta1}. Then, \[P_{V_t} \beta'(x) = \beta(x)\, ,\] where $P_{V_t}$ is the projection matrix onto subspace $V_t$. 
        \item \label{item:beta4} The coefficients sum to one, even though some entries could be negative \[\ones \beta(x) = 1\] where $\one^\top$ is all ones row vector.
    \end{enumerate}
\end{proposition}
\begin{proof}
    First, recall by definition, the coefficients satisfy \Cref{eq:beta1}. The first claim follows from $\spn(V_t^\perp) = \ker(\Pr[F_t|B_t])$ (\Cref{prop:vperp}).
    Finally, the last claim follows by multiplying both sides in \Cref{eq:beta1} by all ones row vector $\ones$ \begin{align*}
        1 = \ones \Pr[F_t|x] =  \ones \Pr[F_t|B_t] \beta(x) = \ones \beta(x)
    \end{align*} where the last equation follows by noting that $\ones \Pr[F_t|x]$ for any probability vector $\Pr[F_t|x]$ is $1$ (recall $F_t$ is set of all futures of length exactly $T-t$).
\end{proof}
\noindent Even though, these are exponentially many coefficients, we next show existence of operators which can be used to construct these coefficients.
\begin{proposition}\label{prop:main1}
    Let $A_{o, t}$ be defined using basis $B_t = \{b_1, \ldots, b_n\}$ and $B_{t+1}$ as: \begin{equation*}
      A_{o, t} =  \begin{bmatrix}
        \beta(b_1 o) & \beta(b_2 o) & \cdots & \beta(b_n o)
      \end{bmatrix} \begin{bmatrix}
        \Pr[o| b_1] & 0           & \cdots & 0           \\
        0           & \Pr[o| b_2] & \cdots & 0           \\
        \vdots      & \vdots      & \ddots & 0           \\
        0           & \cdots      & 0      & \Pr[o| b_n]
      \end{bmatrix}.
    \end{equation*} Then, it satisfies the following: \begin{enumerate}
        \item \label{item:ao1} $span(A_{o,t}) \subset span(V_{t+1})$
        \item $\ker(A_{o,t}) \supset span(V^\perp_{t})$
        \item $A_{o,t}\beta(x) = \Pr[o|x]\beta(x o)$
        \item $\ones A_{o,t} \beta(x) = \Pr[o|x]$
    \end{enumerate}
\end{proposition}
\begin{proof}
    The first two properties follow from the definition of $A_{o,t}$. Let $F_t$ be all observation sequences of length $T-t$. Similar to proof of \Cref{prop:toy:1}, $A_{o,t}$ satisfies \begin{align*}
          \Pr[F_{t+1} | B_{t+1}] A_{o,t} = \Pr[o F_{t+1} | B_{t}] P_{V_t}
    \end{align*} Since $oF_{t+1}$ is a subset of $F_t$, we get by multiplying $\beta(x)$ on both sides, we get
    \begin{align}
        \Pr[F_{t+1} | B_{t+1}] A_{o,t} \beta(x) & = \Pr[o F_{t+1} | B_{t}] \beta(x)\tag{as $\beta(x) \in \spn(V_t)$ } \\
                                                           & = \Pr[o F_{t+1} | x] \tag{as $oF_{t+1}$ is a subset of $F_t$}              \\
                                                           & = \Pr[F_{t+1}| x o] \Pr[o | x] \tag{by Bayes rule}\\
                                                           &= \Pr[F_{t+1}| B_{t+1}] \beta(xo) \Pr[o | x]\notag
    \end{align} By uniqueness of $\beta(x o)$ (\Cref{item:beta2}) and that $\spn(A_{o,t}) \subset \spn(V_{t+1})$ (\Cref{item:ao1}), we get that \begin{align*}
        \frac{A_{o,t} \beta(x)}{\Pr[o|x]} = P_{V_{t+1}}\frac{A_{o,t} \beta(x)}{\Pr[o|x]}  = \beta(x o)
    \end{align*} The last one follows from multiplying both sides above by all ones row vector and then using $\ones \beta(x) = 1$ (\Cref{item:beta4}). \begin{equation*}
        \ones \frac{A_{o,t} \beta(x)}{\Pr[o|x]} = \ones \beta(x o) = 1
    \end{equation*}
\end{proof}

\subsection{Estimating covariance matrix in Frobenius norm}\label{app:sub-1}
In this section, we would show how to estimate the following objects: \[
    q(x) = \Pr[F_t|B_t]^\top  D_t^{-1} \Pr[F|x]\quad \text{and}\quad \Sigma_{B_t} = \Pr[F_t|B_t]^\top D_t^{-1} \Pr[F_t|B_t] \, , 
\] which we need for estimating the operator $A_{o,t}$. We ignore $t$ subscript when clear from context.

\begin{lemma}[Restatement of \Cref{prop:estimate-sum1-re}]\label{prop:estimate-sum1}
    Let $\{B_t\}_{t \in [T]}$ be a basis of distribution $\Pr[\cdot]$. Let $c$ be some upper bound on the coefficients under basis $B_t$ and $n$ be the size of basis $B_t$. Define $s(b^*, x)$ as the following sum where $b^* \in B_t$ and $x$ is a history of length $t$:\[
        s(b^*, x) = \sum_{f \in F_t}  \frac{\Pr[f|b^*]\Pr[f|x]}{d(f)}\, .
    \] Then we can learn estimate $\widehat s(b^*, x)$ such that with probability $1-\delta$, \[
        \abr{s(b^*, x) - \widehat s(b^*, x)} \leq \varepsilon
    \] using at most \[
       m_\varepsilon :=  O\rbr{\frac{c^2 n^{10} |\obs|^2 T^4}{\varepsilon^6} \log^2\rbr{\frac{1}{\delta}}}
    \] conditional samples.
    
\end{lemma}
\begin{proof}
    We start by writing $s(b^*, x)$ in terms of expectation under $\Pr[\cdot|x]$: \begin{align*}
        s(b^*, x)  & = \sum_{f \in F_t}  \frac{\Pr[f|b^*]\Pr[f|x]}{d(f)} = \EE_{f \sim \Pr[\cdot|x]}  \sbr{\frac{\Pr[f|b^*]}{d(f)}}
    \end{align*}  
    
    With this, we define (un-normalized) probability functions $\widehat \Pr[\cdot|b]$ for $b \in B_t$ which set probability to $0$ if $f$ is part of a set $F_{b}$ (to be defined) which will depend only on history $b$ i.e. \[
    \widehat \Pr[f|b]  =  \begin{cases}
            0 & f \in F_b\\
            \Pr[ f | b] & \text{otherwise}
        \end{cases}
    \] 
        We define  $\widehat d$ as mixture distribution of $\widehat \Pr[\cdot | b]$ for $b\in B_t$ i.e. $\widehat d(f) = \EE_{b\in B_t}[\widehat \Pr[f|b]]$.
        An important aspect of our definitions is that for any future $f$, \begin{equation}\label{eq:a11}
            0 \leq \frac{\widehat \Pr[f|b^*]}{\widehat d(f)} \leq n \quad \text{and}\quad 0 \leq \frac{\Pr[f|x]}{d(f)} \leq cn^{3/2}
        \end{equation}are upper bounded.
            Now, suppose $\max_b \Pr[F_b | b] \leq p$. Then, \begin{align*}
        &\abr{\EE_{f \sim \Pr[\cdot|x]}  \sbr{\frac{\Pr[f|b^*]}{d(f)}}- \EE_{f \sim \Pr[\cdot|x]}  \sbr{\frac{\widehat \Pr[f|b^*]}{\widehat d(f)}}}\\
        &\leq \sum_{f} \abr{\frac{\Pr[f|b^*]\Pr[f|x]}{d(f)} - \frac{\widehat \Pr[f|b^*] \Pr[f|x]}{\widehat d(f)}}\\
        &= \sum_{f} \abr{\frac{\Pr[f|b^*]\Pr[f|x]}{d(f)} \pm \frac{\widehat \Pr[f|b^*]\Pr[f|x]}{d(f)} - \frac{\widehat \Pr[f|b^*] \Pr[f|x]}{\widehat d(f)}}\\
        &\leq \sum_{f} \abr{\rbr{\Pr[f|b^*] - \widehat \Pr[f|b^*]} \frac{\Pr[f|x]}{d(f)}} + \sum_{f} \abr{\frac{\widehat \Pr[f|b^*]\Pr[f|x]}{\widehat d(f) d(f)}\rbr{\widehat d(f) - d(f)}}\\
        &\leq cn^{3/2} \Pr[F_{b^*} | b^*] + cn^{5/2} \frac{\sum_{k = 1}^n \Pr[F_{b_k} | b_k]}{n}\\
        &\leq O(cn^{5/2} p)
     \end{align*} where the last step follows from $\max_b \Pr[F_b | b] \leq p$. Now, we can sample $m= (1/2c^2 n^3 p^2) \log(2/\delta)$ random futures from $\Pr[\cdot|x]$ and call this set $S$. Then, again by \Cref{eq:a11} and Hoeffding's inequality (\Cref{prop:hoeffding}), we get that \[ 
        \Pr\sbr{\abr{\EE_{f \sim \Pr[\cdot|x]}  \sbr{\frac{\widehat \Pr[f|b^*]}{\widehat d(f)}} - \EE_{f \sim S}  \sbr{\frac{\widehat \Pr[f|b^*]}{\widehat d(f)}}}\geq cn^{5/2} p} \leq \delta
     \]  Together, with $(1/2c^2 n^3 p^2) \log(2/\delta)$ conditional samples, we get the following guarantee with probability $1-\delta$, \[
        \abr{\EE_{f \sim \Pr[\cdot|x]}  \sbr{\frac{\Pr[f|b^*]}{d(f)}}- \EE_{f \sim S}  \sbr{\frac{\widehat \Pr[f|b^*]}{\widehat d(f)}}} \leq O(cn^{5/2} p)
     \]

    To estimate $\widehat \Pr[f|b^*]$ and $\widehat d(f)$, we need to identify the case when $f$ is ``irregular''. 
    
    \begin{definition}[$\alpha$-regular future]\label{def:ill}
        We define a future $f$ to be  $\alpha$-regular for history $b$
         if for all $\tau \in [t]$ \[
           \Pr[f_{\tau}| b f_{1:\tau - 1}] > \alpha  \, .
        \] Otherwise, $f$ is $\alpha$-irregular for history $b$.
    \end{definition}
    
    To do this, define empirical estimates $\widetilde \Pr[f_\tau | b f_{1:\tau - 1}]$ for every future $f$ and basis $b \in B_t$ using $\widetilde O(n|S|T/\alpha^2 \log(1/\delta))$ many samples. Then, we perform the following test $A(f,b)$ for each future $f$ and basis history $b$ using these estimates: 
    
    \begin{definition}[Test $A(f,b)$]
        Test $A(f,b)$ passes if the empirical estimate $\widetilde \Pr[f_\tau | b f_{1:\tau - 1}] > 2\alpha$ for all $\tau \in [t]$ and fails otherwise.
    \end{definition}
    
     Note that with probability $1-\delta$, (i) if test $A(f,b)$ passes for future $f$ and history $b$, then $f$ is $\alpha$-regular for $b$, and (ii) if test $A(f,b)$ fails for future $f$ and history $b$, then $f$ is $3\alpha$-irregular for $b$. In the rest of the proof, we condition on the event that this relationship between test $A(f,b)$ and irregular futures holds for all futures  $f \in S$ and $b \in B_t$. We set $F_b$ to be the set of futures $f$ which are $3\alpha$-regular for basis $b$ and removing the ones where test $A(f,b)$ passes. By \Cref{prop:irregular-conc}, \begin{equation}\label{eq:a12}
        p = \Pr[F_b| b] \leq O(|\obs| T \alpha) 
     \end{equation}
    
    Now, we define estimates $\widetilde \Pr[f | b]$ for each future $f \in S$ and basis history $b \in B_t$ by first running test $A(f,b)$ on future $f$ and history $b$. If test $A(f,b)$ fails we set $\widetilde \Pr[f | b] = 0$. Note that otherwise $\widetilde \Pr[f | b]$ to be the estimate from \Cref{prop:estimate-conditionals-1} i.e. with probability $1-\delta$, \[
        \abr{\widetilde \Pr[f|b] - \Pr[f|b]} \leq \gamma \Pr[f|b]
    \] This requires $\widetilde O(n|S|T^2/(\gamma \alpha)^2 \log (1/\delta))$ many samples. Moreover, because $\widehat \Pr[f | b^*] = \Pr[f|b^*]$ for futures where tests passes, we can estimate the probability ratios with additive error:
    \[
        \frac{\widetilde \Pr[f|b^*]}{\widetilde d(f)} \leq \frac{(1+\gamma) \widehat\Pr[f|b^*]}{(1-\gamma) \widehat d(f)} \leq \frac{(1+ 4 \gamma)\widehat\Pr[f|b^*] }{\widehat d(f)} \leq \frac{\widehat\Pr[f|b^*] }{\widehat d(f)} +  4 \gamma n  \]
     where the second inequality holds for $\gamma < 1/2$ and \[
        \frac{\widetilde \Pr[f|b^*]}{\widetilde d(f)} \geq \frac{(1-\gamma) \Pr[f|b^*]}{(1+\gamma) d(f)} \geq \frac{(1- 2 \gamma)\Pr[f|b^*] }{d(f)} \geq \frac{\Pr[f|b^*] }{d(f)} -  2 \gamma n
    \] where the second inequality holds for $\gamma < 1/2$. This means \begin{equation}\label{eq:a13}
        \abr{\EE_{f \sim S}  \sbr{\frac{\widehat \Pr[f|b^*]}{\widehat d(f)}} - \EE_{f \sim S}  \sbr{\frac{\widetilde \Pr[f|b^*]}{\widetilde d(f)}}}    \leq 4 \gamma n
    \end{equation} Combining \Cref{eq:a11,eq:a12,eq:a13}, we can build estimate $[\widehat q(x)]_i$ for $[q(x)]_i$ such that with probability $1 - O(\delta)$ \[
        \abr{[\widehat q(x)]_i - [q(x)]_i} \leq \varepsilon
    \] using \[
        \widetilde O(c^2 n^{10} |\obs|^2 T^4 \frac{1}{\varepsilon^6} \log^2(\frac{1}{\delta}))
    \] many conditional samples.
\end{proof}

\begin{corollary}\label{cor:estimatesigma}
    We can learn approximations $\widehat q(b o)$ and $\widehat \Sigma_{B_t}$ for all $b \in B_t$, observations $o$ and time $t \in [T]$ such that with probability $1-\delta$, \begin{align*}
        \norm{\widehat q(b o) - q(b o)}_F \leq \varepsilon \sqrt{n}\, , \norm{\widehat \Sigma_{B_t} - \Sigma_{B_t}}_F &\leq \varepsilon n
    \end{align*} using at most $\widetilde O(n^2 |\obs| T m_\varepsilon)$ conditional samples.
\end{corollary}
\begin{proof}
     Each entry of $\Sigma_{B_t}$ and $q(b o)$ is given by sums of the form $s(b*, x)$ where $b^*$ is a basis history and $x$ is arbitrary history of length $t$. Therefore, we can estimate each of them using estimates given by \Cref{prop:estimate-sum1}: \begin{align*}
        &\norm{\widehat \Sigma_{B_t} - \Sigma_{B_t}}^2_F \leq \sum_{i, j \in [n]} \abr{s(b_i, b_j) - \widehat s(b_i, b_j)}^2 \leq n^2 \varepsilon^2
    \end{align*} The result similarly holds for $q(b o)$. There are $O(n^2)$ entries in each matrix with $T$ many $\Sigma_{B_t}$ matrices and $O(n)$ entries in each vector with $n |\obs| T$ many $q(bo)$ vectors. 
\end{proof}

\begin{proposition}
    \label{prop:estimate-conditionals-1} Consider a future $f_{1:t}$ of length $t$ and history $x$. Fix $\gamma > 0$. Suppose $\Pr[f_{\tau}| x f_{1:\tau-1}] > \alpha$ for each $\tau \in [t]$. Then, we can build estimate $\widehat \Pr[f_{1:t}|x]$ such that with probability $1-\delta$ \[
        \abr{\widehat \Pr[f_{1:t}|x] - \Pr[f_{1:t}|x]} \leq \gamma \Pr[f_{1:t}|x]
    \] using at most $O(t^2/(\gamma \alpha)^2 \log (t/\delta))$ conditional samples.
\end{proposition}
\begin{proof}
    By Hoeffding's inequality, using $m=16 t^2 \log(t/\delta)/(\gamma \alpha)^2$ samples, we have, with probability greater than $1-\delta$, that for all $\tau\in[t]$,
          \[
        \abr{\Pr[f_\tau| x f_{1:\tau-1}] - \widehat \Pr[f_\tau| x f_{1:\tau-1})} 
        \leq \frac{\gamma \alpha}{2t}
        \leq \frac{\gamma}{2t} \Pr[f_\tau| x f_{1:\tau-1}),
        \]
        where the last step uses our assumption above.
        For an upper bound, we have,
        \[
        \widehat\Pr[f_{1:t}| x]=  \Pi_{\tau=1}^t
        \widehat\Pr[f_\tau|x f_{1:\tau-1}] \leq (1+\frac{\gamma}{2t})^t 
        \Pi_{\tau=1}^t \Pr[f_\tau|x f_{1:\tau-1}]
        =(\frac{\gamma}{2t})^t\Pr[f_{1:t} | x]
        \leq \Pr[f_{1:t} | x) + \gamma \Pr[f_{1:t}| x]
        \]
        where the last step follows with $(1+ a)^t \leq 1 + 2at$ for $a < 1/2t$.
        Similarly, for a lower bound we have:
        \[
        \widehat\Pr[f_{1:t} | x]\geq 
        (1-\frac{\gamma}{2t})^t\Pr[f_{1:t}| x) \geq \widehat\Pr[f_{1:t}| x] - \gamma \Pr[f_{1:t}| x],
        \]
        where the last step follows with $(1+ a)^t \geq 1 + at$ for $a \geq -1$.
\end{proof}

\begin{proposition}\label{prop:irregular-conc} Define a future $f$ to be  $\alpha$-irregular for history $b \in B_t$
    if there exists some $\tau \in [t]$ ($\tau$ can depend on $b$) such that \[
      \Pr[f_{\tau}| b f_{1:\tau - 1}] < \alpha  \, .
   \] Let $F_b$ be the set of futures $f$ where  $f$ is $\alpha$-irregular for history $b$. Then, \[
    \Pr[F_b| b] \leq |\obs| T \alpha
   \]
\end{proposition}
\begin{proof}
    Let future $f$ be of length $T$. We first partition the set $F_b$ into $T$ sets: $F_{b,1}, \ldots, F_{b,T}$ based on the first time irregular is observed i.e. $f \in F_{b,t} \iff t = \min_{\tau} \Pr[f_{\tau}| b f_{1:\tau - 1}] < \alpha$. Now, \begin{align*}
        &\sum_{f_{1:T} \in F_{b,t}} \Pr[f | b] \\
        &= \sum_{f_{1:t-1} f_t f_{t + 1: T} \in F_{b,t}} \Pr[f_{t +1 : T} | b f_{1:t-1} f_t] \Pr[f_t | b f_{1:t-1} ] \Pr[f_{1:t-1}|b]\\
        &\leq \sum_{f_{1:t-1} f_t \in F_{b,t}} \Pr[f_t | b f_{1:t-1} ] \Pr[f_{1:t-1}|b]\sum_{\text{futures $g$ of length $T-t-1$}} \Pr[g | b f_{1:t-1} f_t] \\
        &\leq |\obs| \alpha  \sum_{f_{1:t-1}  \in F_{b,t}} \Pr[f_{1:t-1}|b] \tag{as $\Pr[f_t | b f_{1:t-1} ] \leq \alpha$}\\
        &\leq |\obs| \alpha
    \end{align*}
    Summing over all $T$ of these sets gives the claim.
\end{proof}

\section{Examples}
\label{sec:examples}
In this section, we show that our parity with noise and all previously known positive results: full rank HMMs from \citep{mossel2005learning,hsu2012spectral} and overcomplete HMMs from \citep{sharan2017learning} can be learned by our algorithm using conditional sampling oracle. Note that we will use the alternate form of fidelity which is more amenable to analysis given by \begin{align*}
    \sigma_+\rbr{ S_t^{\frac{1}{2}}\Pr[F_t | H_t]^\top D_{t}^{-1} \Pr[F_t | H_t] S_t^{\frac{1}{2}}}
    = \sigma_+ \rbr{D_t^{-1/2} \EE_{x_{1:t} \sim \Pr[\cdot]}\sbr{ \Pr[F_t|x_{1:t}] \Pr[F_t|x_{1:t}]^\top} D_t^{-1/2}} 
\end{align*} 

\subsection{HMMs induce low rank distributions}\label{sec:hmmlowr}
Before we show that Hidden Markov Models induce low rank distributions, we clarify a subtle issue in our definition of rank. 
\begin{remark}\label{rem:rank}
    If $\Pr[h] = 0$ for some history $h$, then we can not use Bayes rule $\Pr[f|h] \Pr[h] = \Pr[h f]$ to define the conditional distribution $\Pr[\cdot | h]$. Moreover, in this case, any consistent setting of $\Pr[\cdot | h]$ i.e. which satisfies $\Pr[of|h] = \Pr[f|ho]\Pr[o|h]$ will induce the same joint distribution $\Pr[\cdot]$, and might have possibly different $\rank(\Pr[\obs^{\leq T-t}|\obs^{t}])$. 
    For Hidden Markov Models, if $\Pr[h] = 0$ for some history $h \in \obs^t$, this can be resolved by assuming $\Pr[\textbf{s}_{t+1} = s | h] = 1/S$ for $s \in \mathcal{S}$. We will show below that for this choice, rank of HMM is at most the number of hidden states.
\end{remark}

\begin{proposition}
    The distribution induced by HMM with $S$ hidden states has rank $\leq S$.
\end{proposition}
\begin{proof}
    We will show that $\rank(\Pr[\obs^{\leq T-t}|\obs^{t}]) \leq S$ under the choice that if $\Pr[h] = 0$ for some history $h \in \obs^t$, then $\Pr[\textbf{s}_{t+1} = s | h] = 1/S$ for $s \in \mathcal{S}$. Define  $\mathcal{S}^+ \subset \mathcal{S}$ be hidden states $s$ for which $\Pr[\textbf{s}_{t+1} = s] > 0$. Define $\Pr[\obs^{\leq T-t} | \mathcal{S}^+]$ to be a matrix whose $((x_{t + \tau}, \ldots, x_{t+1}),s)$ entry is given by probability $\Pr[x_{t + \tau}, \ldots, x_{t+1} | \textbf{s}_{t+1} = s]$. Similarly define $\Pr[\mathcal{S}^+ | H_t]$ to be a matrix whose $(s, (x_{t}, \ldots, x_{1}))$ entry is given by probability $\Pr[\textbf{s}_{t+1} = s | x_{t}, \ldots, x_{1}]$. Then, $\Pr[\obs^{\leq T-t} | H_t]$ can be decomposed into $\Pr[\obs^{\leq T-t} | H_t] = \Pr[\obs^{\leq T-t} | \mathcal{S}^+] \Pr[\mathcal{S}^+ | H_t]$, and therefore $\rank(\Pr[\obs^{\leq T-t} | H_t]) \leq S$.
\end{proof}
\subsection{Parity with noise}
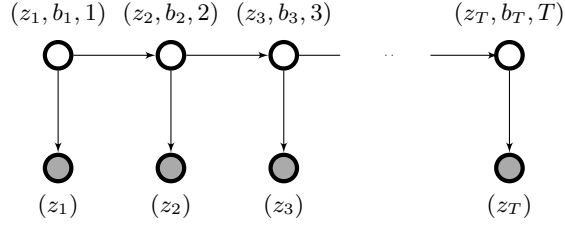
\begin{figure}[t!]
    \centering
        \begin{tikzpicture}
        \node[state] (1) {};
        \node[state, right of=1] (2) {};
        \node[state, right of=2] (3) {};
        \node[state, right of=3, xshift=1.5cm] (n) {};
        \node[state, fill={rgb:black,1;white,2}, below of=1] (1b) {};
        \node[state, fill={rgb:black,1;white,2}, below of=2] (2b) {};
        \node[state, fill={rgb:black,1;white,2}, below of=3] (3b) {};
        \node[state, fill={rgb:black,1;white,2}, below of=n] (nb) {};
        \node[below of=1b, yshift=+1cm] {\footnotesize$(z_1)$};
        \node[below of=2b, yshift=+1cm] {\footnotesize$(z_2)$};
        \node[below of=3b, yshift=+1cm] {\footnotesize$(z_3)$};
        \node[below of=nb, yshift=+1cm] {\footnotesize$(z_T)$};
        \draw
        (1) node[above= 0.25] {\footnotesize $(z_1, b_1, 1)$}
        (2) node[above= 0.25] {\footnotesize $(z_2, b_2, 2)$}
        (3) node[above= 0.25] {\footnotesize $(z_3, b_3, 3)$}
        (n) node[above= 0.25] {\footnotesize $(z_T, b_{T}, T)$};
        \draw[edge] (1) to (2);
        \draw[edge] (2) to (3);
        \draw[] (3) edge ($(3)!0.25!(n)$);
        \draw[dotted] ($(3)!0.45!(n)$)--($(3)!0.50!(n)$);
        \draw[edge] ($(3)!0.65!(n)$) to ($(3)!0.95!(n)$);
        \draw[edge] (1) to (1b);
        \draw[edge] (2) to (2b);
        \draw[edge] (3) to (3b);
        \draw[edge] (n) to (nb);
    \end{tikzpicture}
  \caption{\label{fig:noisy-parity} Hidden Markov model for noisy parity. Each hidden state is of the form $(z_t, b_t, t)$ where $z_t$ represents the current bit to be output, $b_t$ is the parity of a secret subset of previous bits and $t$ is the bit position. $b_1$ is always set to  $0$ and $z_T$ is set to $b_T$ with probability $\alpha$ and $1 - b_T$ otherwise for some $\alpha \in (0,1/2)$. For other positions $t\in [T-1]$, transition from hidden state $(z_t, b_t, t)$ goes uniformly randomly to hidden states $(1,b_{t+1}, t+1)$ and $(0,b_{t+1}, t+1)$, where $b_{t+1} = b_{t} \oplus z_t$ if $t \in I$ and $b_{t+1} = b_t$ otherwise.}
  \end{figure}
We first formally define the distribution induced by parity with noise which
has been extensively studied in the computational learning theory \citep{blum1994weakly}.

\begin{definition}[Parity with noise]
    \label{def:noisy-parity}
    Let $(x_1,\ldots,x_{T-1})$ be a vector in $\{0,1\}^{T-1}$, $S$ a subset of $[T-1]$ and $0 < \alpha < 1/2$. The parity of $(x_1,\ldots,x_{T-1})$ on $S$ is the boolean function $\phi(x_1,\ldots,x_{T-1})$ which outputs $0$ if the number of ones in the sub-vector $(x_i)_{i \in S}$ is even and $0$ otherwise. Then the distribution induced by HMM for parity with noise is such that the first $T-1$ bits are uniform over $\{0,1\}^{T-1}$ and the last bit is $\phi(x_1,\ldots,x_{T-1})$ with probability $1-\alpha$ and $1 - \phi(x_1,\ldots,x_{T-1})$ otherwise.
\end{definition}

We now show that Parity with noise HMM satisfies the conditions of \Cref{thm:main}. Note that for $\alpha = 1/2$, each bit becomes a random bit. And then as can be seen in the proof below, the fidelity is $1$ (since the second eigenvalue goes to $0$).
\begin{proposition}
    The distribution induced by Parity with noise HMM has rank $\leq 2T$ and fidelity $(1-2\alpha)^2/ 2$ under a basis of size $\leq 2$ for every sequence length $t\in T$.
\end{proposition}
\begin{proof}
The claim about rank follows from noting that $\Pr[\cdot|x] = \Pr[\cdot|y]$ is same if both $x$ and $y$ have the same length $t$; and the subvectors $(x_i)_{i\in S \cap [t]}$ and $(y_i)_{i\in S \cap [t]}$ have the same number of ones modulo $2$. 

We only need to show that the distribution has large fidelity. The proof is same for all $t\in [T]$, so we prove this for a particular $t$. For a fixed $t$, $\Pr[F_t|x]$ only depends on the parity of the secret subset $(x_i)_{S\cap [t]}$, so there are only two options for $\Pr[F_t|x]$. Let those be $v_1$ and $v_2$. We choose the basis $B_t$ to be any histories with probability vector $v_1$ and $v_2$. Note that if the last bit in the future $f$ matches the parity, then the corresponding probability entry is $(1-\alpha)/2^{T-t-1}$, otherwise it is $\alpha/2^{T-t-1}$. Also, we have $v_1(f) = (1-\alpha)/2^{T-t-1}$ if and only if $v_2(f) = \alpha/2^{T-t-1}$. Therefore, $d_t(f) = 1/2^{T-t}$ for every future $f \in F_t$.

Let $V$ be the matrix with columns given by $v_1$ and $v_2$. Our goal is to show that following matrix has large non-zero eigenvalues \[
    D_t^{-1/2} \EE\sbr{ \Pr[F_t|x] \Pr[F_t|x]^\top} D_t^{-1/2} = \frac{1}{2} D_t^{-1/2} V V^\top D_t^{-1/2} .
\] Since $D_t^{-1/2} V V^\top D_t^{-1/2} $ has same eigenvalues as $V^\top D_t^{-1} V$, we compute its eigenvalues. The diagonal entries of $V^\top D_t^{-1} V$ are $\alpha^2 + (1-\alpha)^2$ and the off-diagonal entries are $2 \alpha(1-\alpha)$. Therefore, the eigenvalues of $V^\top D_t^{-1} V$ are $1$ and $(1- 2\alpha)^2$. This gives us a lower bound on $\Delta \geq (1- 2\alpha)^2/2$.
\end{proof}

\subsection{Full rank HMMs}
We first define full rank HMMs studied in \citep{mossel2005learning,hsu2012spectral}. Recall the definition of HMMs (\Cref{def:hmm}).

\begin{definition}[Full rank HMMs]
    We say an HMM is full rank, if its emission matrix $\mathbb{O}$ and state transition matrix $\mathbb{T}$ have full column rank.
\end{definition}

We next show the distribution induced by full rank HMM can be learned by our algorithm in \Cref{thm:main}. Let $P_{2,1}$ be an $O \times O$ matrix with $(i,j)$th entry $\Pr[o_2= i, o_1 = j]$. Note that the previous result \citep{hsu2012spectral} in this setting depend on smallest eigenvalues of $P_{2,1}$.

\begin{proposition}\label{prop:fullrank}
    The distribution induced by full rank HMM has rank $S$ and fidelity $\sigma_{\min}(P_{2,1})^2$ under a basis of size $\leq O$ for every sequence length $t\in T$.
\end{proposition}
\begin{proof}
We note that in the full rank case, we can simplify our algorithm considerably. By our assumptions, $\rank(\Pr[\obs | \obs]) = S$. And therefore instead of picking all distributions, we can replace $\Pr[F_t|x]$ by one-step probabilities $\Pr[\obs|x]$. We can do this because $\Pr[\obs|x] = \Pr[\obs|\obs] \beta(x)$ implies $\Pr[F_t|x] = \Pr[F_t|\obs] \beta(x)$ as $\rank(\Pr[F_t|H_t]) = \rank(\Pr[F_t|\obs]) = \rank(\Pr[\obs|\obs]) = S$ and $\obs\subset F_t$ where $H_t$ and $F_t$ is the set of all observation sequences of length $t\geq 1$ and $\leq T-t$ respectively. This also means we do not need a different basis for each $t\in [T]$.
     And only need to show that the following matrix has large eigenvalues: 
\[
    D^{-1/2} \EE_{o \sim \Pr[\cdot]}\sbr{ \Pr[\obs|o] \Pr[\obs|o]^\top} D^{-1/2}   
\] where $D$ is a diagonal matrix with $d(o') = \EE_{o \in \obs}[\Pr[o'| o]]$. Since, eigenvalues of $D^{-1/2}$ are $\geq 1$, we are interested in eigenvalues of \[
    \EE_{o \sim \Pr[\cdot]}\sbr{ \Pr[\obs|o] \Pr[\obs|o]^\top} = \Pr[\obs|\obs] U \Pr[\obs|\obs]^\top
\] where $U$ is a diagonal matrix of size $|\obs| \times |\obs|$ with its diagonal entries given by $\Pr[o]$. Then, by definition, $P_{2,1} = \Pr[\obs|\obs] U$. Using this we can  lower bound $\Delta$ by $\sigma_{\min}(P_{2,1})^2$.


\end{proof}

\subsection{Overcomplete HMMs}
We first define the class of overcomplete HMMs which can be learned using techniques in \citep{sharan2017learning}. In \citep{sharan2017learning}, the authors were concerned with the stationary distribution induced by HMMs. To define their assumptions, let $S$ be the number of hidden states and $\tau = O(\lceil \log_{|\obs|} S \rceil)$. Moreover, let $H_\tau$ be the set of histories of length $\tau$ of the form $x_{-\tau},\ldots, x_{-1}$, $F_\tau$ be the set of histories of length $\tau$ of the form $x_{0}, \ldots, x_{\tau}$, and $\mathcal{S} = \{1,\ldots, S\}$ be the set of hidden states. Note that by our setting of $\tau$, $H_\tau$, $F_\tau$ and $\mathcal S$ are all size $O(S)$. Define $\Pr[F_\tau | \mathcal{S}]$ to be a matrix of size $|F_\tau| \times S$ whose $((x_{0}, \ldots, x_{\tau}),s)$ entry is given by probability $\Pr[x_{0}, \ldots, x_{\tau} | s_0 = s]$. Similarly define $\Pr[H_\tau | \mathcal{S}]$ as the equivalent matrix for time-reversed Markov chain whose $((x_{-\tau}, \ldots, x_{-1}),s)$ entry is given by probability $\Pr[x_{-\tau}, \ldots, x_{-1} | s_0 = s]$.

\citet{sharan2017learning} showed efficient algorithms for HMMs under assumptions which imply (a) $\Pr[F_\tau|\mathcal{S}]$ and $\Pr[H_\tau|\mathcal{S}]$ matrices are rank $S$, (b) the condition number of $\Pr[F_\tau|\mathcal{S}]$ and $\Pr[H_\tau|\mathcal{S}]$ is $\poly(S)$ and (c) every hidden state has stationary probability at least $1/\poly(S)$.

We next show the distribution induced by these HMMs can be learned by our algorithm in \Cref{thm:main}. We will not require the uniqueness of columns of $\Pr[\obs, \mathcal{S}]$.

\begin{proposition}
    The distribution induced by HMMs defined above has rank $S$ and fidelity $(\poly(S))^{-1}$ under a basis of size $O(S)$ for every sequence length $t\in T$.
\end{proposition}
\begin{proof}
Just like the full rank case (\Cref{prop:fullrank}), we can simplify our algorithm considerably. We choose $B$ to be the set $H_\tau$. Let $F$ be the set of all observation sequences of length $T-\tau$ for some $T > 2\tau$. Now, we can replace $\Pr[F|x]$ by probabilities $\Pr[F_\tau|x]$ in our algorithm. We can do this because $\Pr[F_\tau|x] = \Pr[F_\tau|B] \beta(x)$ implies $\Pr[F|x] = \Pr[F|B] \beta(x)$ as $\rank(\Pr[F|B]) = \rank(\Pr[F_\tau|B]) = S$ and $F_\tau \subset F$. Second, we do not need a different basis for each $t\in [T]$ as $\Pr[F_\tau|x]$ lives in the span of $\Pr[F_\tau|B]$ for every history $x$ as $\rank(\Pr[F_\tau|B]) = \rank(\Pr[F_\tau|H]) = S$ and $B\subset H$ where $H$ is the set of all histories of length $\leq T$.
     This means we only need to show that the following matrix has large eigenvalues: 
\[
    D^{-1/2} \EE_{x_{1:\tau} \sim p}\sbr{ \Pr[F_\tau|x_{1:\tau}] \Pr[F_{\tau}|x_{1:\tau}]^\top} D^{-1/2}   
\] where $D$ is a diagonal matrix with entries $d(f) := \EE_{b \in B} \Pr[f | b]$ on the diagonal. Since, eigenvalues of $ D^{-1/2}$ are $\geq 1$, we are interested in eigenvalues of \[
    \EE_{x_{1:\tau} \sim p}\sbr{ \Pr[F_\tau|x_{1:\tau}] \Pr[F_{\tau}|x_{1:\tau}]^\top} = \Pr[F_\tau|B] K \Pr[F_\tau|B]^\top
\] where $K$ is a diagonal matrix of size $|B| \times |B|$ with diagonal entries given by $k(b) = \Pr[b]$. Define $\Pr[F_\tau H_\tau]$ to be a matrix of size $|F_\tau| \times |H_\tau|$ whose $((x_{0}, \ldots, x_{\tau}),(x_{-\tau}, \ldots, x_{-1}))$ entry is given by probability $\Pr[x_{-\tau}, \ldots, x_{-1},x_{0}, \ldots, x_{\tau}]$.
Then, by definition, $\Pr[F_\tau|B] K = \Pr[F_\tau H_\tau]$. Using this, each entry of $K < 1$ and every hidden state has stationary probability at least $1/\poly(S)$, we can lower bound $\Delta$ by $\sigma_{\min}(\Pr[F_\tau H_\tau])^2 = 1/\poly(S)$.

\end{proof}

\section{General algorithm for finding approximate basis}
\label{sec:approx-ellip}
In this section, we learn an approximate version of basis for probability vectors. Throughout this section, we ignore $t$ subscript in $F_t$ and $B_t$ when clear from context. We define an approximate basis, which allows us to ignore histories which have very low probability under the distribution $\Pr[\cdot]$:
\begin{definition}[Approximate Basis]
    Fix $0 < \varepsilon< 1$. For a distribution $\Pr[\cdot]$ over observation sequences of length $T$, we say a subset of observations sequences $B$ forms an $\varepsilon$-basis for $\Pr[\cdot]$ at length $t\in [T]$, if for every observation sequence $x = (x_1, \ldots, x_t)$, there exists coefficients $\beta(x)$ with $\ell_2$ norm $\norm{\beta(x)}_2 \leq 1$ such that:
    \[
        \EE_{x \sim \Pr[\cdot]}\sbr{\norm{\Pr[F|x]  - \Pr[F|B] \beta(x)}_1}\leq \varepsilon \, .\]
\end{definition}

We first define regular distributions.
\begin{definition}[Regular distribution]
  We say a distribution $\Pr[\cdot]$ is $\alpha$-regular if $\min \Pr[o| x] \geq \alpha$ where the minimum is over all histories $x$ and observations $o$ where $\Pr[o| x] \neq 0$.
\end{definition}
We now present the main result in this section: how to build an approximate basis for a regular low rank distribution.
\begin{theorem}
    \label{lemma:approx-1}
Let $\Pr[\cdot]$ be an $\alpha$-regular distribution over observation sequences of length $T$ with rank $r$. Fix $0< \varepsilon <  \alpha/T r$ and $0 < \delta < 1$. Then, in $\poly(r, T, 1/\varepsilon, 1/\alpha, \log(1/\delta))$ time, with probability $1-\delta$, we can find an $\varepsilon$-basis of size at most $O(r^2 T^3 \log(1/\alpha \varepsilon))$ using conditional sampling oracle.
\end{theorem}
We believe the regularity assumption on the distribution can be removed using the ideas from \Cref{app:sub-1} but leave it as future work.

\subsection{Learning coefficients}
Towards this goal, we first show how to check given an observation sequence $x$ and set of observation sequences $B$ if there exists $\beta(x)$ such that \[
    \norm{\Pr[F|x]  - \Pr[F|B] \beta(x)}_1 \leq \varepsilon
\]

Instead of directly working with the $\ell_1$ loss, we first define an $\ell_2$ loss which we can use as its proxy. 

\begin{definition}
    For set of observation sequences $B = \{b_1, \ldots, b_h\}$, observation sequence $x$, column vector $\beta \in \R^{|B|}$, we define the $\ell_2$ approximation error as:
\[
    L_{B, x}(\beta) := \EE_{f\sim d}
    \bigg[\bigg( \frac{\Pr[f|x]}{d(f)} - \sum_{j = 1}^h \beta_j \frac{\Pr[f|b_j]}{d(f)} \bigg)^2\bigg], ,
    \] where $d$ is the mixture distribution for $B$:
    \[
      d(f) = \frac{1}{2}  \Pr[f|x] +
      \frac{1}{2h} \sum_{i=1}^{h} \Pr[f|b_i],
    \] When clear from context, we drop the $B,x$ superscript.
\end{definition}
We will use our ability to simulate relative probabilities for regular distributions (\Cref{prop:estimate-conditionals-1}) to build our guess for approximate basis.
\begin{proposition}
    \label{prop:approx-1}
Let $\Pr[\cdot]$ be an $\alpha$-regular distribution, $x$ be an observation sequence of length $t \in [T]$ and $B$ be any set of observation sequences of length $t$. Suppose $f_1, \ldots f_m$ are i.i.d. samples from $d$. Then, using $\poly(T, 1/\varepsilon, 1/\alpha, \log(1/\delta))$ many conditional samples, we can have estimates $\widehat \Pr[f_i|b]$ for
all $i\in [m]$ and $b \in B \cup \{x\}$ such that with probability $1-\delta$, \[
    \sup_{\|\beta\|_2 \leq C, \widehat L_{B,x}(\beta) \leq \widehat L_{B,x}(0)}  \big| L_{B,x}(\beta) - \widehat L_{B,x}(\beta)\big|
    \leq \varepsilon
    \] where the estimated $\ell_2$ error function $\widehat L_{B,x}$ is defined as \begin{align*}
        \widehat L_{B,x}(\beta) &:=\frac{1}{m} \sum_{i\in[m]}\bigg( \frac{\widehat \Pr[f_i|x]}{\widehat d(f_i)} 
  - \sum_{j=1}^h \beta_j \frac{\widehat \Pr[f_i|b_j]}{\widehat d(f_i)} \bigg)^2.
    \end{align*} and $\widehat d(f_i)$ is the mixture distribution
    defined with the estimated probabilities.
\end{proposition}
\begin{proof}
    For notational convenience, we will drop the $B,x$ superscript and simply write $L(\cdot)$ in the proof.
    Using $\poly(1/\gamma, 1/\alpha, T, \log(1/\delta))$ conditional samples (\Cref{prop:estimate-conditionals-1}), with probability $1-\delta$, we can have estimates $\widehat \Pr[f_i|b]$ such that, for
    all $i\in [m]$ and $b \in B \cup \{x\}$,
  \[
  |\Pr[f_i|b]-\widehat \Pr[f_i|b]| \leq \gamma \Pr[f_i|b].
  \]
    Define:
    \begin{align*}
  \overline L(\beta) &:= \frac{1}{m}\sum_{i\in[m]}\bigg( \frac{ \Pr[f_i|x]}{ d(f_i)} 
  - \sum_{j=1}^h \beta_j \frac{ \Pr[f_i|b_j]}{ d(f_i)} \bigg)^2.
    \end{align*}
    We have that:
    \[
    \big| L(\beta) - \widehat L(\beta)\big|
     \leq \big| L(\beta) - \overline L(\beta)\big|
     +\big| \overline L(\beta) - \widehat L(\beta)\big|.
     \]
     We will handle the two terms above separately.
    
    To bound $| L(\beta) - \overline L(\beta)|$, let
    us first show that for any observation sequence $f$:
    \[
    \sum_{j=1}^h \bigg(\frac{ \Pr[f|b_j]}{ d(f)}\bigg)^2 \leq 4h^2
    \]
    To see this, observe that:
    \[
    \max_{i\in[m],j\in[h]}\left| \frac{ \Pr[f_i|b_j]}{ d(f_i)}\right|
      \leq 2h,
    \]
    which implies:
    \[
    \sum_{j=1}^h \bigg(\frac{ \Pr[f|b_j]}{ d(f)}\bigg)^2\leq 
    2h \sum_{j=1}^h \frac{ \Pr[f|b_j]}{ d(f)} =4h^2.
    \]
    Now using that the square loss is a $2$-smooth function and a standard uniform convergence argument, we have that
    \[
    \sup_{\|\beta\|_2 \leq C}  \big| L(\beta) - \widehat L(\beta)\big|
    \leq 16 (C h+1) \sqrt{\frac{\log(1/\delta)}{m}},
    \]
    holds with probability greater than $1-\delta$.

    For the second term, define
    \[
    \Delta(f_i) =\frac{ \Pr[f_i|x]}{ d(f_i)} 
    - \sum_{j=1}^h \beta_j \frac{ \Pr[f_i|b_j]}{ d(f_i)}
    -\bigg(
      \frac{\widehat \Pr[f_i|x]}{\widehat d(f_i)} 
    - \sum_{j=1}^h \beta_j \frac{\widehat \Pr[f_i|b_j]}{\widehat d(f_i)}\bigg).
    \]
    Let us first show that, for all $i\in[m]$,
    \[
     | \Delta(f_i)| \leq 4\gamma h (1+C\sqrt{h}).
    \]
    We have that:
    \begin{align*}
      |\Delta(f_i)|
    & \leq \left| \frac{ \Pr[f_i|x]}{ d(f_i)} 
    -\frac{\widehat \Pr[f_i|x]}{\widehat d(f_i)}\right|
    + \sum_{j\in[h]} \beta_j
    \left| \frac{ \Pr[f_i|b_j]}{ d(f_i)} 
    -\frac{\widehat \Pr[f_i|b_j]}{\widehat d(f_i)}\right|\\
    &\leq (1+\|\beta\|_1) \max_{b \in B \cup \{x\}}
    \left| \frac{ \Pr[f_i|b]}{ d(f_i)} 
    -\frac{\widehat \Pr[f_i|b]}{\widehat d(f_i)}\right|\\
    & \leq
    (1+C\sqrt{h}) \max_{b \in B \cup \{x\}}
    \left| \frac{ \Pr[f_i|b]}{ d(f_i)} 
    -\frac{\widehat \Pr[f_i|b]}{\widehat d(f_i)}\right|.  
    \end{align*}
    The claim would follow provided we have that, for all $i\in[m]$ and all $b \in B \cup \{x\}$,
    \begin{equation}\label{eq:second_term}
    \left| \frac{ \Pr[f_i|b]}{ d(f_i)} 
    -\frac{\widehat \Pr[f_i|b]}{\widehat d(f_i)}\right| \leq
    4\gamma h, 
    \end{equation}
    To see this, using that
    $\gamma\leq 1/2$, we have the following upper bound
    that:
    \begin{align*}
    \frac{\widehat \Pr[f_i|b]}{\widehat d(f_i)}
    \leq \frac{(1+\gamma) \Pr[f_i|b]}{(1-\gamma) d(f_i)}
    \leq \frac{(1+4\gamma) \Pr[f_i|b]}{ d(f_i)}
    \leq \frac{\Pr[f_i|b]}{ d(f_i)} + 4\gamma h
    \end{align*}
    and the lower bound:
    \begin{align*}
    \frac{\widehat \Pr[f_i|b]}{\widehat d(f_i)}
    \geq \frac{(1-\gamma) \Pr[f_i|b]}{(1+\gamma) d(f_i)}
    \geq \frac{(1-2\gamma) \Pr[f_i|b]}{ d(f_i)}
    \geq \frac{\Pr[f_i|b]}{ d(f_i)} - 2\gamma h.
    \end{align*}
    This completes the proof of \Cref{eq:second_term}.
    
    Now consider any $\beta$ such that $\widehat L(\beta) \leq \widehat
    L(0)$. Also, it is straightforward that $\widehat L(0)\leq
    4$. This implies that:
    \begin{align*}
     \overline L(\beta) - \widehat L(\beta)
    & = 
    \frac{2}{m}\sum_{i\in[m]} \bigg(
      \frac{\widehat \Pr[f_i|x]}{\widehat d(f_i)} 
    - \sum_{j=1}^h \beta_j \frac{\widehat \Pr[f_i|b_j]}{\widehat d(f_i)}\bigg) \Delta(f_i)
    +\frac{1}{m}\sum_{i\in[m]} \Delta(f_i)^2
    \\
      & \leq
        2\sqrt{\widehat L(\beta) \cdot \frac{1}{m}\sum_{i\in[m]} \Delta(f_i)^2}
    +\frac{1}{m}\sum_{i\in[m]} \Delta(f_i)^2\\
    & \leq
        4\sqrt{ \frac{1}{m}\sum_{i\in[m]} \Delta(f_i)^2}
    +\frac{1}{m}\sum_{i\in[m]} \Delta(f_i)^2\\
      & \leq
        4\cdot 4\gamma h (1+C\sqrt{h}) 
    +16\gamma^2 h^2 (1+C\sqrt{h})^2\\
      & \leq
    16\gamma h (1+\gamma h) (1+C\sqrt{h})^2.
    \end{align*} where the first step follows from $a^2 - b^2 = 2b(a-b) + (a-b)^2$ and the second step from Cauchy–Schwarz inequality.
    Combining the bounds for the first and second terms completes the proof.
    \end{proof}

\subsection{Algorithm}
We now ready to present our algorithm. The user furnishes $\varepsilon$, the accuracy with which approximate basis is to be learned; and $\delta$, a confidence parameter. The parameter $n$ and $H$ depend on the input. 

\begin{algorithm}[h!]
    \SetArgSty{textrm}
    \caption{Learning approximate basis using conditional samples.}
    \label{alg:approx-1}
    \SetAlgoLined
    \For{round $i = 1,2, \ldots, H$}{
        Sample $n$ samples $x = (x_1,\ldots, x_t)$ of length $t$ from distribution $\Pr[\cdot]$.\\
        Check using \Cref{prop:approx-1} if any of the samples $x$ above is a ``counterexample'' i.e. satisfies \[
          \min_{\beta \in \R^h, \norm{\beta}_2 \leq C} \widehat L_{B, x}(\beta) > \frac{\varepsilon^2}{8}
        \]\\
        \If{we find such a counterexample $x = (x_1, \ldots, x_t)$}{
            Add $x$ to $B$
        }
        \Else{
            return $B$
        }
    }
\end{algorithm}

By Hoeffding's inequality, it is clear that if this algorithm ends then we have found an approximate basis. We know show that with high probability, this algorithm ends in small number of rounds. 
\begin{proposition}\label{prop:approx-3}
    Let $\Pr[\cdot]$ be an $\alpha$-regular distribution over observation sequences of length $T$ with rank $r$. Fix $0 \leq \varepsilon \leq \alpha^2/T^2 r^2$. Let $C = \sqrt{2Tr \log(1/\alpha \varepsilon)}$ and let $H$ be any natural number provided $H\geq 8 rT^2 \log(1/\varepsilon \alpha)$. Consider any sequence of observation sequences $b_1, b_2, \ldots, b_H$. Let $B_h = \{b_1,\ldots, b_h\}$. Then, there exists $h \leq H$ such that:\[
        \min_{\beta \in \R^h, \norm{\beta}_2 \leq C} L_{B_h, b_{h+1}}(\beta) \leq \varepsilon
        \]
\end{proposition}

\begin{proof}
 For $\beta \in \R^h$, define:
\[
L^{(h)}_\lambda(\beta) :=L_{B_{h}, b_{h+1}}(\beta)
+\lambda \sum_{j=1}^h \beta_j^2
\] and $d^{(h)}$ to be the mixture distribution corresponding to $B_h$.
Define:
\[
\overline \Pr[f|b_j] := \frac{\Pr[f|b_j]}{\sqrt{d^{(h)}(f)}} .
\]
It will be helpful to overload notation and view $P_{j} = \Pr[F|b_j]$ as
a vector of length $|F|$ and $P_{1:h} = \Pr[F|B_h]$ as a matrix of size
$|F|\times h$, whose columns are
$ P_1, \ldots P_h$. We overload notation analogously for the vector $\overline
P_j$ and the matrix $\overline P_{1:h}$.
Also, let $D^{(h)}$ be a diagonal
matrix of size $|F|\times |F|$, whose diagonal entries are
$d^{(h)}$, where we drop the $h$ superscript when clear
from context. With our notation, we have that
$\overline P _{h+1}=D^{-1/2} P_{h+1}$ and
$\overline P_{1:h}=D^{-1/2} P_{1:h}$. We will write $P$ and $\overline P$
in lieu of $P_{1:h}$ and $\overline P_{1:h}$, when clear from
context. 

We have:
\begin{align*}
  L^{(h)}(\beta) &= \EE_{x\sim d^{(h)} }
\bigg[\bigg( \frac{ \Pr[f|b_{h+1}]}{ d^{(h)}(f)} 
- \sum_{j=1}^h \beta_j \frac{ \Pr[f|b_j]}{ d^{(h)}(f)} \bigg)^2\bigg]\\
&=(D^{-1} P_{h+1}- D^{-1}P_{1:h} \beta) ^\top D (D^{-1} P_{h+1}-  D^{-1}P_{1:h} \beta)\\
&= \|\overline P_{h+1}- \overline P \beta\|^2,
\end{align*}
where we have used our matrix notation. Let us consider the following ridge regression estimator:
\begin{align*}
\beta^{(h)}_\lambda &= 
\argmin_{\beta \in \R^d} \bigg(L^{(h)}(\beta) + \lambda \|\beta\|^2\bigg)\\
 &=\argmin_{\beta \in \R^d} \bigg(  \|\overline P_{h+1}- \overline P \beta\|^2
+ \lambda \|\beta\|^2\bigg)\\
 &=(\overline P^\top \overline P
+\lambda I)^{-1}\overline P^\top  \overline P_{h+1}.
\end{align*}
For $p_{\textrm{min}} = \alpha^T$, define $\Sigma_h $ to be the $|F|\times |F|$ sized matrix,
as follows
\[
\Sigma_h := p_{\textrm{min}} \lambda  I + P_{1:h} P_{1:h}^\top
= p_{\textrm{min}} \lambda I + \sum_{j=1}^h P_j P _j^\top.
\]
We will now show that:
\begin{equation}\label{eq:loss_bound}
\min_{\beta \in \R^h} L^{(h)}_\lambda(\beta)
\leq \lambda P_{h+1}^\top \Sigma_h^{-1}  P_{h+1}.
\end{equation}
Define:
\[
\overline \Sigma_h := \lambda I + \overline  P_{1:h} \overline P_{1:h}^\top.
\]
One can verify that:
\begin{align*}
\overline P \beta_\lambda^{(h)}
= \overline P (\overline P^\top \overline P
+\lambda I)^{-1}\overline P^\top  \overline P_{h+1}
=\overline P \overline P^\top (\overline P \overline P^\top 
+\lambda I)^{-1} \overline P_{h+1}
=\overline P \overline P^\top \overline \Sigma_h^{-1} \overline P_{h+1},
\end{align*}
and that:
\[
\|\beta^{(h)}_\lambda\|^2 = \overline P_{h+1}^\top \overline \Sigma_h^{-1}\overline P \overline P^\top
\overline \Sigma_h^{-1}\overline P_{h+1}
= \overline P_{h+1}^\top \overline \Sigma_h^{-1} (\overline\Sigma_h - \lambda I)
\overline \Sigma_h^{-1}\overline P_{h+1}
= \overline P_{h+1}^\top \overline \Sigma_h^{-1}\overline P_{h+1}
-\lambda \overline P_{h+1}^\top \overline \Sigma_h^{-2} \overline P_{h+1}.
\]
Using this, we have:
\begin{align*}
\min_{\beta \in \R^h} L^{(h)}_\lambda(\beta)
&= \|\overline P_{h+1}- \overline P \beta^{(h)}_\lambda\|^2+\lambda \|\beta^{(h)}_\lambda\|^2\\
&= \bigg\| \bigg(I-\overline P \overline P^\top \overline \Sigma_h^{-1}\bigg)
  \overline P_{h+1}\bigg\|^2
+\lambda \|\beta^{(h)}_\lambda\|^2\\
&= \bigg\| \bigg(\overline P \overline P^\top 
+\lambda I-\overline P \overline P^\top \bigg)
\overline \Sigma_h^{-1}  \overline P_{h+1}\bigg\|^2
+\lambda \|\beta^{(h)}_\lambda\|^2\\
&=\lambda^2 \overline P_{h+1}^\top\overline \Sigma_h^{-2}  \overline P_{h+1}
+\lambda \|\beta^{(h)}_\lambda\|^2\\
&= \lambda \overline P_{h+1}^\top \overline \Sigma_h^{-1}  \overline P_{h+1}.
\end{align*}
By our assumption on $p_{\textrm{min}}$, we have that
$d^{(h)}(f) \geq p_{\textrm{min}}$, which implies $D \succeq \frac{1}{p_{\textrm{min}}}  I$.
Using this, we have
\begin{align*}
\lambda \overline P_{h+1}^\top \overline \Sigma_h^{-1}  \overline
  P_{h+1}  
&= \lambda (D^{-1/2} P_{h+1})^\top (\lambda I +  D^{-1/2} P P^\top  D^{-1/2})^{-1}  D^{-1/2} P_{h+1}\\
&= \lambda P_{h+1} (\lambda D +  P P^\top )^{-1}  P_{h+1}\\
& \leq \lambda P_{h+1} (p_{\textrm{min}} \lambda I +  P P^\top)^{-1}  P_{h+1}\\
& = \lambda P_{h+1}^\top \Sigma_h^{-1}  P_{h+1},
\end{align*}
where we have used the definition of $\Sigma_h$ in the last step. This
proves the claim in~\Cref{eq:loss_bound}.

This implies:
\[
\min_{h\leq H} \min_{\beta \in \R^h} L^{(h)}_\lambda(\beta)\leq
  \frac{1}{H} \sum_{h=1}^H \min_{\beta \in \R^h} L^{(h)}_\lambda(\beta) \leq 
\frac{\lambda}{H} \sum_{h=1}^H P_{h+1}^\top \Sigma_h^{-1}  P_{h+1}.
\]
Using that $\| P_{h+1}\|^2\leq 1 $ (since $P_{h+1}$ is a probability distribution), 
Lemma~\ref{lemma:elliptical} implies:
\[
  \frac{1}{H} \sum_{h=1}^H
  \log(1+ P_{h+1}^\top \Sigma_h^{-1}  P_{h+1})
\leq \frac{r}{H} \log\big(1+H /(p_{\textrm{min}}\lambda)\big) ,
\]
which implies there exists an $h \leq H$ such that:
\[
\log(1+ P_{h+1}^\top \Sigma_h^{-1}  P_{h+1})
\leq \frac{r}{H} \log\big(1+H /(p_{\textrm{min}}\lambda) \big) .
\]
For $H\geq 4 r \log(1+ H /(p_{\textrm{min}}\lambda))$, exponentiating leads to:
\[
P_{h+1}^\top \Sigma_h^{-1}  P_{h+1}
\leq \exp\bigg(\frac{r}{H} \log\big(1+H /(p_{\textrm{min}}\lambda) \big) \bigg)-1
\leq 2 \frac{r}{H} \log\big(1+H /(p_{\textrm{min}}\lambda)\big)
\]
where the last step follows due to our choice of $H$. This shows that there exists an $h\leq H$ such that:
\[
\min_{\beta \in \R^h} L^{(h)}_\lambda(\beta) \leq
\frac{\lambda r}{H } 
\log\big(1+ H /(p_{\textrm{min}}\lambda) \big) .
\] Choosing $\lambda = \varepsilon^2$, setting $H = 8 rT^2 \log(1/\varepsilon \alpha)$ suffices to satisfy our assumptions. This implies we get the minimum loss is achieved at $\beta$ whose norm is bounded by\[
    \norm{\beta}_2 \leq C := \sqrt{2Tr \log\rbr{\frac{1}{\alpha \varepsilon}}}
\]
and therefore for $\varepsilon < \alpha^2/T^2r^2$, we get 
\[
    \min_{\beta \in \R^h, \norm{\beta} \leq C} L^{(h)}(\beta) \leq \varepsilon
\]
\end{proof}
The following is a variant of the Elliptical Potential Lemma, from
the analysis of linear
 bandits~\citep{dani2008stochastic}.
 
 \begin{lemma}[Elliptical potential]\label{lemma:elliptical}
   Consider a sequence of vectors $\{x_1,\dots, x_{T}\}$ where, 
   for all $i\in [T]$,
  each $x_i \in \Vcal$, where $\Vcal$ is a $d$-dimensional
  subspace of a Hilbert space, and  $\|x_i\|\leq B$.  Let $\lambda \in \mathbb{R}^+$.
   Denote 
   $\Sigma_t = \Sigma_0 +
   \sum_{i=1}^{t} x_i x_i^{\top}$.
   We have that:
 \begin{align*}
 \min_{i\in[T]} \ln\left( 1 + x_i^\top \Sigma_{i}^{-1} x_i  \right) \leq
 \frac{1}{T}\sum_{i=1}^{T} \ln\left( 1 + x_i^\top \Sigma_{i}^{-1} x_i  \right)
   = \frac{1}{T} \ln\frac{\det\left( \Sigma_T \right)  }{ \det(\lambda I)}
\leq \frac{d}{T} \log \bigg(1+\frac{TB^2}{d\lambda}\bigg)
 \end{align*}
 \end{lemma}
We now finish the proof of \Cref{lemma:approx-1} by adding the missing details.
We first show that $\ell_1$ loss is upper bounded by our $\ell_2$ proxy loss.
\begin{proposition}
    \label{prop:approx-2}
    Let  $x$ be an observation sequence of length $t \in [T]$ and $B = \{b_1, \ldots, b_h\}$ be any set of observation sequences of length $t$. We have that:
    \[
    \norm{\Pr[F|x] - \sum_{j=1}^h \beta_j \Pr[F|b_j]}_1 \leq \sqrt{L_{B, x}(\beta)}.
    \]  
    \end{proposition}
    \begin{proof}
      We have that:
    \begin{align*}
        \norm{\Pr[F|x] - \sum_{j=1}^h \beta_j \Pr[F|b_j]}_1 &= \EE_{f\sim d }
    \left[\ \left| \frac{ \Pr[f|x]}{ d(f)} 
    - \sum_{j=1}^h \beta_j \frac{ \Pr[f|b_j]}{ d(f)} \right|\ \right]\\
    &\leq 
    \sqrt{\EE_{f\sim d }
    \bigg[\bigg( \frac{ \Pr[f|x]}{ d(f)} 
    - \sum_{j=1}^h \beta_j \frac{ \Pr[f|b_j]}{ d(f)} \bigg)^2\bigg]},
    \end{align*}
    where the last step uses Jensen's inequality.
    \end{proof}
\begin{proof}[Proof of \Cref{lemma:approx-1}]
  We choose $C = \sqrt{2Tr \log(16/\alpha \varepsilon^2)}$, $H = 8 rT^2 \log(16/\varepsilon^2 \alpha)$ and $n = O(1/\varepsilon^2 \log(H/\delta))$. Let $\beta_h(x)$ be the coefficients such that \[
    \beta_h(x) = \argmin_{\beta \in \R^h, \norm{\beta}_2 \leq C} \widehat L_{B_h, x}(\beta)
  \]
  From $\Cref{prop:approx-1}$, using $\poly(r, T, 1/\varepsilon, 1/\alpha, \log(1/\delta))$ many conditionally samples, we can get with probability $1-\delta/2$, for all $h \in [H]$ and observation sequence $x$ in our random sample \begin{equation}\label{eq:helpapprox1}
       \big| L_{B_h,x}(\beta(x)) - \widehat L_{B_h,x}(\beta(x))\big|
      \leq \frac{\varepsilon^2}{32}
    \end{equation}
  In our algorithm, when we find a counterexample, it means for some observation sequence $x$ in the random sample: \[
        \widehat L_{B_h, x}(\beta(x)) > \frac{\varepsilon^2}{8}
      \]
      This means, by \Cref{eq:helpapprox1}, for that observation sequence $x$,
      \[
          L_{B_h, x}(\beta(x)) > \frac{\varepsilon^2}{16}.
        \]
        However, by our choice of $C$ and $H$, by \Cref{prop:approx-3} this can not happen $H$ times and therefore, our algorithm should end in at most $H$ rounds.\\

        \noindent We will now show that the overall error of our basis is small. When our algorithm ends, then we should have for all observation sequence $x$ in our random sample: \[
          \min_{\beta \in \R^h, \norm{\beta}_2 \leq C} \widehat L_{B_h, x}(\beta) \leq \frac{\varepsilon^2}{8}
        \] This means by \Cref{eq:helpapprox1}, we should have for all observation sequence $x$ in our random sample:\[
          \min_{\beta \in \R^h, \norm{\beta}_2 \leq C} L_{B_h, x}(\beta) \leq \frac{\varepsilon^2}{4}.
        \] By \Cref{prop:approx-2}, this implies for all observation sequence $x$ in our random sample:
        \[
          \min_{\beta \in \R^h, \norm{\beta}_2 \leq C} \norm{\Pr[F|x]  - \Pr[F|B] \beta}_1 \leq \frac{\varepsilon}{2}.
        \]
        Therefore, for our choice of $n$, by Hoeffding inequality, we get with probability $1-\delta/2$, \[
          \Pr_{x \sim p}\sbr{\min_{\beta \in \R^h, \norm{\beta}_2 \leq C} \norm{\Pr[F|x]  - \Pr[F|B] \beta}_1 >  \varepsilon/2} \leq \varepsilon/2
        \] Since, for all histories $x$ \[
          \min_{\beta \in \R^h, \norm{\beta}_2 \leq C} \norm{\Pr[F|x]  - \Pr[F|B] \beta }_1 \leq 1 , 
        \] we get that \[
          \EE_{x \sim p}\sbr{\min_{\beta \in \R^h, \norm{\beta}_2 \leq C} \norm{\Pr[F|x]  - \Pr[F|B] \beta(x)}_1} \leq (1-\varepsilon/2) \varepsilon/2 + \varepsilon/2 \leq \varepsilon
        \] Let $B'$ be the set where we repeat $C$ times every $b \in B$, then we get that \[
          \EE_{x \sim p}\sbr{\min_{\beta \in \R^h, \norm{\beta}_2 \leq 1} \norm{\Pr[F|x]  - \Pr[F|B] \beta(x)}_1} \leq (1-\varepsilon/2) \varepsilon/2 + \varepsilon/2 \leq \varepsilon
        \] This completes the proof.
\end{proof}

\section{Helper propositions}
\begin{proposition}[Hoeffding's inequality]\label{prop:hoeffding}
    Let $X_1, X_2, \ldots, X_n$ be independent random variables such that $a\leq X_{i}\leq b$ almost surely. Consider the sum of these random variables,\[S_n = X_1 + \cdots + X_n.\]
    Then for all $t > 0$,
    \[ \Pr\sbr{\abr{S_{n}-\EE \sbr{S_{n}}}\geq nt} \leq 2\exp \rbr{-{\frac {2n t^{2}}{(b-a)^{2}}}}
    \]
    Here $E[S_n]$ is the expected value of $S_n$.
\end{proposition}

In our work, we care about how far are the projection operators onto top eigenspace for two symmetric matrices which are close to each other. This follows as a corollary of Davis-Kahan theorem.

\begin{proposition}[Davis-Kahan theorem]\label{thm:davis-kahan}
    Let $\Sigma$ and $\widehat \Sigma$ be real symmetric matrices with the eigenvalue decomposition: $V \Lambda_0 V^\top + V^\perp \Lambda_1 V^{\perp \top}$ and $\widehat V \widehat \Lambda_0 \widehat V^\top + \widehat V^\perp \widehat \Lambda_1 \widehat V^{\perp \top}$. If the eigenvalues of $\Lambda_0$ are contained in an interval $[a,b]$, and the eigenvalues of $\widehat \Lambda_1$ are excluded from the interval $[a-\gamma, b + \gamma]$ for some $\gamma > 0$, then \[
        \norm{\widehat V^{\perp \top} V}_F \leq \frac{\norm{\widehat V^{\perp \top} (\widehat \Sigma - \Sigma) V}_F}{\gamma}
    \]
\end{proposition}

\begin{corollary}[\citep{projections}]
    \label{cor:project}
    Let $\Sigma$, $\widehat \Sigma$, $V$, $\widehat V$ and $\gamma$ be as defined above. Assume $V, \widehat V \in \R^{n \times r}$. Then, \[
\norm{VV^\top - \widehat V \widehat V^\top}_F \leq \frac{\sqrt{2r} \norm{\Sigma - \widehat \Sigma}_2}{\gamma}
    \]
\end{corollary} 
\end{document}